\documentclass[10pt,journal]{IEEEtran}
\ifCLASSOPTIONcompsoc
  \usepackage[nocompress]{cite}
\else
  \usepackage{cite}
\fi
\ifCLASSINFOpdf
\else
\fi

\usepackage[compact]{titlesec}
\usepackage{subfigure}
\usepackage{graphicx}

\usepackage{amsmath,amsthm}
\usepackage{amssymb,wasysym}
\usepackage{mathrsfs}
\usepackage{cite}
\usepackage{color}
\usepackage{url}
\usepackage{bbm}

\usepackage{wrapfig}
\usepackage{verbatim}
\usepackage{dsfont}
\usepackage{amsmath}
\usepackage{algorithm}
 \usepackage{algorithmic}
\usepackage{xurl}

\newtheorem{theorem}{\textbf{Theorem}}
\newtheorem{lemma}{\textbf{Lemma}}
\newtheorem{corollary}{\textbf{Corollary}}

\newtheorem{assumption}{\textbf{Assumption}}
\renewcommand{\algorithmicrequire}{\textbf{Input:}}  % Use Input in the format of Algorithm
\renewcommand{\algorithmicensure}{\textbf{Output:}} % Use Output in the format of Algorithm
\allowdisplaybreaks[3]

\ifodd 0
\newcommand{\rev}[1]{{\color{blue}#1}} %revise of the text
\else
\newcommand{\rev}[1]{#1}
\fi

\begin{document}

\title{AdaptSFL: Adaptive Split Federated Learning in Resource-constrained Edge Networks}

\author{Zheng Lin, Guanqiao Qu, Wei Wei, Xianhao Chen,~\IEEEmembership{Member,~IEEE}, Kin K. Leung,~\IEEEmembership{Fellow,~IEEE}
\thanks{Z. Lin, G. Qu, W. Wei, and X. Chen are with the Department of Electrical and Electronic Engineering, University of Hong Kong, Pok Fu Lam, Hong Kong, China (e-mail: linzheng@eee.hku.hk; gqqu@eee.hku.hk; weiwei@eee.hku.hk; xchen@eee.hku.hk).}
\thanks{K. K. Leung is with the Electrical and Electronic Engineering Department, and the Computing Department, Imperial College, SW7 2BT London, U.K. He was partly sponsored by the EPSRC EP/Y037242/1. (e-mail: kin.leung@imperial.ac.uk).}
% \thanks{\textit{(Corresponding author: Xianhao Chen)}}
}

\markboth{Journal of \LaTeX\ Class Files,~Vol.~14, No.~8, August~2015}%
{Shell \MakeLowercase{\textit{et al.}}: Bare Advanced Demo of IEEEtran.cls for IEEE Computer Society Journals}


\IEEEtitleabstractindextext{
\begin{abstract}
The increasing complexity of deep neural networks poses significant barriers to democratizing AI to resource-limited edge devices. To address this challenge, split federated learning (SFL) has emerged as a promising solution that enables device-server co-training through model splitting. However, although system optimization substantially influences the performance of SFL, the problem remains largely uncharted. In this paper, we first provide a unified convergence analysis of SFL, which quantifies the impact of model splitting (MS) and client-side model aggregation (MA) on its learning performance, laying a theoretical foundation for this field. Based on this, we introduce AdaptSFL, a resource-adaptive SFL framework to accelerate SFL under resource-constrained edge computing systems. Specifically, AdaptSFL adaptively controls MS and client-side MA to balance communication-computing latency and training convergence. Extensive simulations across various datasets validate that our proposed AdaptSFL framework takes considerably less time to achieve a target accuracy than existing benchmarks.
\end{abstract}

}

\maketitle

% is not selected <OR> if conference mode is selected - because compsoc
% conference papers position the abstract like regular (non-compsoc)
\IEEEdisplaynontitleabstractindextext

% \fi
\IEEEpeerreviewmaketitle

% 1.5 pages
\section{Introduction}\label{Intro}
With the increasing prevalence of edge devices such as smartphones, laptops, and wearable gadgets, a vast amount of data is being generated at the network edge. It is forecasted that there will be 50 billion IoT devices, generating 79.4 Zettabytes of data in 2025~\cite{IDC}. This unprecedented volume of data unlocks the potential of artificial intelligence (AI), leading to significant advancements in pivotal domains such as smart health, intelligent transportation, and natural language processing~\cite{2022Letaief,zhu2020toward,fang2024automated,lin2022channel,hu2024collaborative,fang2024ic3m}.

Status quo model training frameworks rely on centralized learning (CL), where a cloud server gathers raw data from participating devices to train a machine learning (ML) model. However, transferring enormous amounts of data from edge devices to the cloud server leads to significant latency, huge bandwidth costs, and severe privacy leakage~\cite{chen2022end,liu2023optimizing,deng2022actions,chen2022federated}. To tackle this issue, federated learning (FL)~\cite{mcmahan2017communication,konevcny2016federated} has been proposed as the answer for privacy-enhancing distributed training. In FL, all participating devices perform model training with locally residing data in parallel and synchronize their models by exchanging model parameters rather than raw data with a parameter server (or fed server), thereby enhancing data privacy~\cite{lin2024fedsn,hu2024accelerating}. Unfortunately, while FL overcomes the limitations of CL, implementing FL 
poses a significant challenge as ML models scale up~\cite{lin2023pushing}. {For instance, the recently prevalent Gemini Nano-2 model, an on-device large language model, comprises 3.25 billion parameters (3GB for 32-bit floats), making on-device training, as required by FL, extremely difficult for resource-constrained edge devices~\cite{team2023gemini}.} 
%In addition, since each device sends intermediate results with smaller data sizes instead of the entire ML model parameters, SL potentially reduces communication overhead compared with FL. 
%, where the devices and server exchange intermediate results (i.e., activations and activations' gradients) at the cut layer during model training

To mitigate the dilemma in FL, split learning (SL)~\cite{vepakomma2018split} has emerged as an alternative privacy-enhancing distributed training paradigm, which enables a server to take over the major workload from clients via model splitting, thus relieving the client-side computing burden~\cite{lin2023split,lin2024hierarchical,wei2025pipelining}. Moreover, since each client sends intermediate results at the cut layer instead of raw data or entire ML model parameters, SL preserves data privacy while reducing communication overhead compared to FL if the model size is large~\cite{lyu2023optimal,lin2025hsplitlora,lin2025leo}. Despite these advantages, the original SL, named vanilla SL~\cite{vepakomma2018split}, adopts a sequential training procedure from one device to another, resulting in excessive training time. As a variant of SL, split federated learning (SFL)~\cite{thapa2022splitfed} overcomes this issue by enabling parallel training among edge devices. As shown in Fig.~\ref{AdaptSFL}, apart from model splitting, SFL also features periodic server-side and client-side sub-model aggregation following the principle of FL. By amalgamating the advantages of SL and FL, SFL has attracted significant attention from academia and industry in recent years. For instance, Huawei deems SFL as a key learning framework for 6G edge intelligence~\cite{huawei2019}. 

While SFL, akin to any other learning framework, demands intensive communication-computing resources,
system optimization for SFL is still in its nascent stage. In particular, model splitting (MS) and client-side model aggregation (MA) decisions heavily influence model accuracy and total training latency of SFL in resource-constrained edge networks. {Although many prior works~\cite{you2023aifed,wei2020federated,wang2019adaptive} have made efforts to optimize global model aggregation frequencies in FL, SFL exhibits entirely different convergence behavior because server-side sub-models and client-side sub-models in SFL can be aggregated at \textit{varied frequencies}, making convergence analysis and optimization non-trivial.} Specifically, server-side sub-models, which are co-located on a server, should \textit{always be synchronized} without communication costs\footnote{This is because more frequent model aggregation generally leads to better training convergence as demonstrated by both theoretical analysis and experiments~\cite{wang2019adaptive}.}, whereas client-side sub-model aggregation generates substantial communication traffic from edge devices, and therefore, should be performed \textit{less frequently}. To demonstrate the impact of client-side MA, Fig.~\ref{fig:motivation_1} shows that aggregation interval $I$ (client-side sub-models are aggregated every $I$ training rounds) balances the converged accuracy and communication overhead: A smaller $I$ usually leads to higher converged accuracy while incurring more communication costs, implying client-side MA should be carefully controlled. {On the other hand, MS also plays a vital role in SFL as it not only affects the computing and communication workloads but also directly impacts the convergence speed of the training process.} To show the effect of MS, Fig.~\ref{fig:motivation_2} reveals that a shallow cut layer (i.e., smaller $L_c$) enhances training accuracy since a larger portion of the model serves as the server-side sub-model that can be synchronized in every training round. However, this may result in increased communication overhead for transmitting cut-layer activations/gradients because a shallower layer in CNNs often involves a larger dimension of layer output. Overall, optimizing MS and client-side MA requires fundamental tradeoff analysis to accelerate SFL under communication and computing constraints.

%Fig.~\ref{sfig:motivation_2_cut_comput_commu} illustrates how MS balances computing workload and communication overhead, and Fig.~\ref{sfig:motivation_2_different_cut} presents the impact of MS on converged accuracy and time.
\begin{figure}[t!]
\centering
\includegraphics[width=8.9cm]{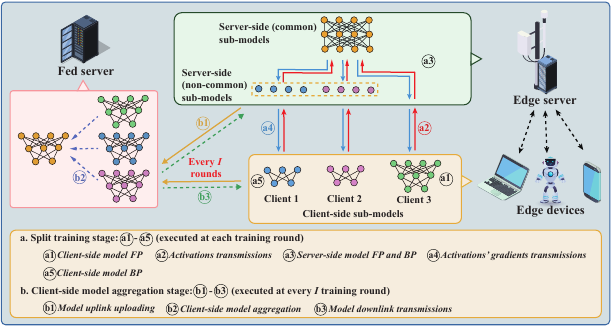}
\vspace{-0.5cm}
\caption{The illustration of AdaptSFL over edge computing systems.
}% \textcolor[rgb]{1.00,0.00,0.00}{edge clouds  at or linked to the SBSs through optical fiber}
\label{AdaptSFL}
\end{figure}

To fill the above void, this paper aims to minimize the training latency of SFL for achieving target learning performance under resource-constrained edge computing systems. To this end, we propose AdaptSFL, a novel resource-adaptive SFL framework with adaptive client-side MA and MS under communication-computing resource constraints. To establish the foundation, we first derive a convergence bound of AdaptSFL that characterizes the effect of client-side MA and MS on its training convergence. This derivation differs from the analysis of FL because client-side and server-side sub-models in SFL can be aggregated at different frequencies. Based on this, we formulate and solve the optimization problem of MS and client-side MA to minimize the total latency under training loss and resource constraints. The main contributions of this paper are summarized as follows.

%To guide the design, we derive a convergence bound of AdaptSFL as a basis for optimizing the above decision variables. Note that the convergence analysis of AdaptSFL differs from FL because client-side and server-side sub-models should be aggregated at different frequencies. 
\begin{itemize}
\item We investigate the system optimization for SFL over resource-constrained edge systems. We theoretically analyze the convergence bound of AdaptSFL, providing some intriguing insights that can lay the foundation for SFL optimization. To our best knowledge, this is the first convergence analysis of SFL that characterizes the effect of MS and client-side MA on its training convergence.
\item {We leverage the derived convergence bound to formulate a joint optimization problem for client-side MA and MS to minimize the total training time for model convergence under resource-constrained edge computing systems.}
\item We decompose the optimization problem into two tractable subproblems of client-side MA and MS, and develop efficient algorithms to obtain the optimal solutions for them, respectively. Subsequently, we develop an alternating optimization method to obtain an efficient suboptimal solution from the joint problem.
\item  We conduct extensive simulations across various datasets to validate our theoretical analysis and demonstrate the superiority of the proposed solution over the counterpart without optimization.
\end{itemize}

\begin{figure}[t]
\vspace{-.5ex}
\setlength\abovecaptionskip{3pt}
\centering
\subfigure[Test accuracy versus epochs.]{
    \includegraphics[height=3.7cm,width=4.13cm]{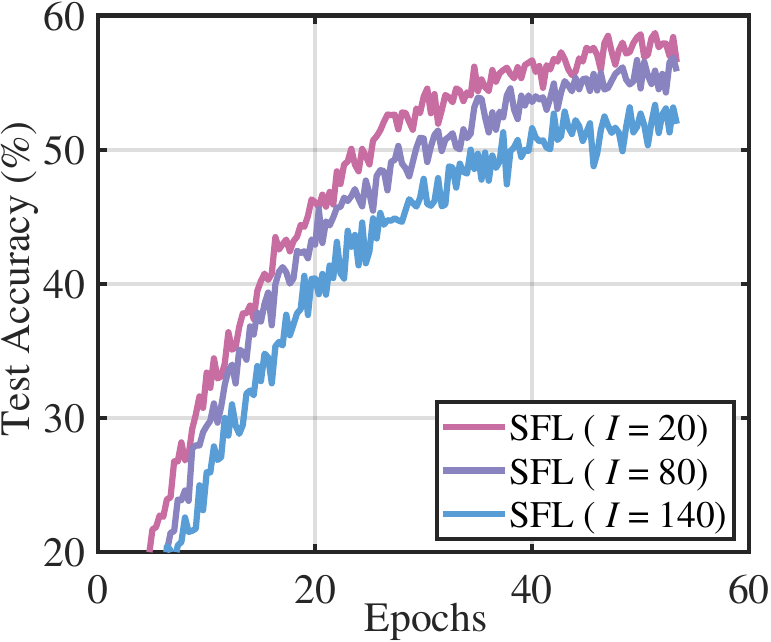}
    \label{sfig:motivation_1_different_I}
}
% \hspace{.1cm}
\subfigure[Test accuracy versus rounds of model aggregation.]{
    \includegraphics[height=3.7cm,width=4.12cm]{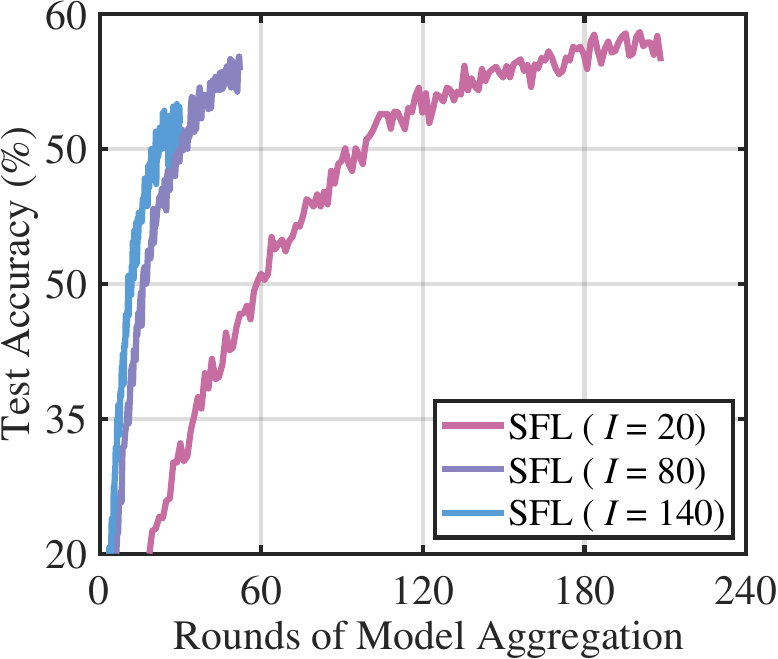}
    \label{sfig:moti_1_accuracy_communication_overhead}
}
    \caption{ The impact of client-side MA on training performance and communication overhead. 
    Fig.~\ref{sfig:motivation_1_different_I} and Fig.~\ref{sfig:moti_1_accuracy_communication_overhead} show the performance for test accuracy versus number of epochs and rounds of model aggregation. The experiment is conducted on the CIFAR-10 dataset under the non-IID setting with $L_c=8$.}
    \label{fig:motivation_1}
    \vspace{-2ex}
\end{figure}

\begin{figure}[t]
\vspace{-.5ex}
\setlength\abovecaptionskip{3pt}
\centering
\subfigure[Computing and communication overhead.]{
    \includegraphics[height=3.68cm,width=4.23cm]{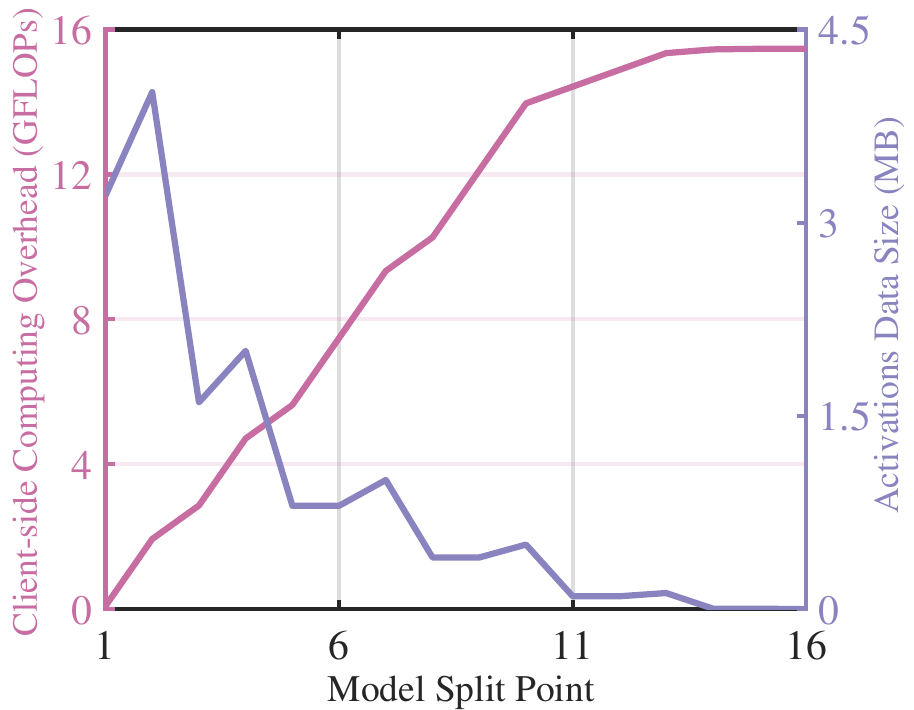}
    \label{sfig:motivation_2_cut_comput_commu}
}
% \hspace{.1cm}
\subfigure[Test accuracy versus epochs.]{
    \includegraphics[height=3.7cm,width=4.02cm]{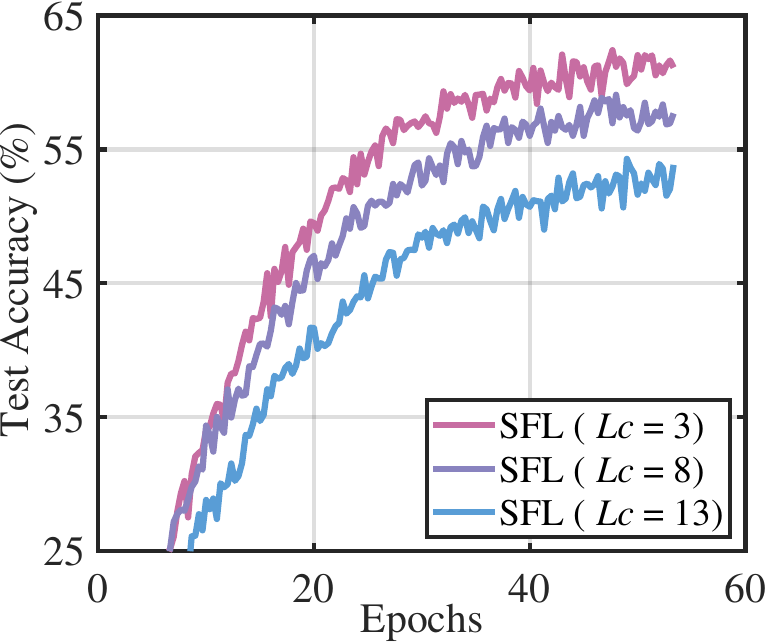}
    \label{sfig:motivation_2_different_cut}
}
    \caption{ The impact of MS on computing and communication overhead as well as training performance. 
    Fig.~\ref{sfig:motivation_2_cut_comput_commu} shows the computing and communication overhead of SFL at different model split points.  Fig.~\ref{sfig:motivation_2_different_cut} presents the performance for test accuracy versus the number of epochs. The experiment is conducted on the CIFAR-10 dataset under the non-IID setting with $I=15$. }
    \label{fig:motivation_2}
    \vspace{-2ex}
\end{figure}

The remainder of this paper is organized as follows. Section~\ref{Rel_Work} introduces related work. Section~\ref{sec_sfl} elaborates on system model and AdaptSFL framework. Section~\ref{convergence_adaptSFL} provides the convergence analysis. We formulate the optimization problem in Section~\ref{prob_formu} and offer the corresponding solution approach in Section~\ref{solu_appro}. Section~\ref{simu_results} provides the simulation results. Finally, concluding remarks are presented in Section~\ref{conclu}.

\section{Related Work}\label{Rel_Work}
% time minimization of FL
% model splitting/SL or split inference
% why lack SFL research
% 1. lack convergence analysis 2. cut layer affect convergence

Tremendous research efforts have been made to improve the learning performance of distributed ML by controlling the interval of model aggregation~\cite{mcmahan2017communication,wang2019adaptive,wu2023faster,2018dont,yu2019computation,xiang2023federated}. 
% These works can be categorized into two directions: heuristic-driven design~\cite{pal2021server,liao2024parallelsfl,guo2021lightfed} and convergence-based design~\cite{mcmahan2017communication,wang2019adaptive,wu2023faster,2018dont,yu2019computation,xiang2023federated}. \rev{The heuristic-driven design empirically optimizes model aggregation without convergence guarantees~\cite{pal2021server,liao2024parallelsfl,guo2021lightfed}. Pal \textit{et al.}~\cite{pal2021server} propose a server-side local gradient averaging mechanism and adaptive learning rate acceleration to enhance the scalability and efficiency of SL. Liao \textit{et al.}~\cite{liao2024parallelsfl} mitigates system heterogeneity by clustering workers and assigning different local update frequencies to each cluster, which improves both training efficiency and model accuracy. Guo \textit{et al.}~\cite{guo2021lightfed} designs a lightweight MA method that selectively transmits critical parameters while dynamically reconstructing the global model, effectively reducing communication overhead and enhancing data privacy.} 
The design often lies in convergence analysis, which serves as the basis for optimizing model aggregation. McMahan \textit{et al.}~\cite{mcmahan2017communication} develop a pioneering FedAvg algorithm that enables each participant to conduct multiple local stochastic gradient descent iterations in a communication round, which aims to reduce the global synchronization frequency without considerably compromising training accuracy. Wang \textit{et al.}~\cite{wang2019adaptive} propose a control algorithm that determines the best tradeoff between local update and global parameter aggregation to minimize the loss function under a given resource budget. Wu \textit{et al.}~\cite{wu2023faster} devise an adaptive momentum-based FL algorithm to achieve the optimal MA interval, which lowers the time consumed for achieving the target accuracy. Smith \textit{et al.}~\cite{2018dont} and Yu \textit{et al.}~\cite{yu2019computation} empirically design dynamic batch size schemes to reduce MA interval to establish a more communication-efficient FL framework. Xiang \textit{et al.}~\cite{xiang2023federated} propose an adaptive epoch adjustment method to reduce communication rounds by dynamically adjusting MA interval of end devices. Nevertheless, the aforementioned schemes do not apply to SFL as server-side sub-models, which are co-located at the same place, can always be synchronized in SFL.

On the other hand, MS heavily impacts training latency of SL, as it leads to varied computing workloads between devices and the server, varied communication overhead due to the diverse output data sizes across layers, and varied convergence speed dependent on the proportion of the server-side common model. Several studies have delved into determining the optimal model split point for SL~\cite{wu2023split,wang2021hivemind,lin2023efficient}. Wu \textit{et al.}~\cite{wu2023split} develop a cluster-based parallel SL framework and a joint cut layer selection, device clustering, and subchannel allocation optimization algorithm to minimize the training overhead considering device heterogeneity and network dynamics. Wang \textit{et al.}~\cite{wang2021hivemind} present a time-efficient multi-split ML framework, which enables real-time optimal split decisions among multiple computing nodes and dynamically adapts split decisions in response to instantaneous network dynamics. Lin \textit{et al.}~\cite{lin2023efficient} propose an efficient parallel SL framework featuring last-layer gradient aggregation, lowering the dimension of activations’ gradients and server-side computing workload. Based on this framework, a joint resource management and MS strategy has been proposed to jointly optimize subchannel allocation, power control, and cut layer selection to minimize the training latency. Unfortunately, these works fail to characterize the effect of MS on training convergence.

In a nutshell, the problem of optimizing MA and MS to enhance SFL convergence has not been investigated. The problem is non-trivial for two reasons: First, a unified convergence analysis that quantifies the impact of client-side MA and MS on the convergence of SFL is still lacking. Second, considering the heterogeneous bandwidth and computing resources of edge devices, MS and client-side MA need to be jointly optimized to accelerate training convergence, resulting in a challenging optimization problem.

\rev{In addition to MS and MA optimization, it is worth mentioning that there are research efforts to improve the resource efficiency of SFL from other perspectives. For instance, Pal \textit{et al.}~\cite{pal2021server} propose a server-side local gradient averaging mechanism and adaptive learning rate acceleration to enhance the scalability and efficiency of SL. Liao \textit{et al.}~\cite{liao2024parallelsfl} mitigate system heterogeneity by clustering workers for split training and assigning different local update frequencies to each cluster, thus improving training efficiency and model accuracy. Guo et al. Guo \textit{et al.}~\cite{guo2021lightfed} design a selective parameter transmission method for model aggregation while dynamically reconstructing the global model to enhance communication efficiency and data privacy. However, these innovative frameworks often lack theoretical convergence guarantees. In contrast, our scheme focuses on vanilla SFL and provides a fundamental trade-off analysis of MA and MS under resource-constrained edge networks.}

% Similarly, Yu \textit{et al.}~\cite{yu2019computation} first derives the convergence rate of FL, and then based on this, develops an exponentially increasing batch size strategy to reduce the global update frequency, thereby improving communication efficiency.
\section{The Resource-adaptive SFL Framework}\label{sec_sfl}
{In this section, we present the system model in Section~\ref{Sys_Model} and the proposed resource-adaptive SFL framework, named AdaptSFL, in Section~\ref{subsec_adapt}.}

\subsection{System Model}\label{Sys_Model}
In this section, we introduce the system model of AdaptSFL to provide a theoretical foundation for the following sections. As illustrated in Fig.~\ref{AdaptSFL}, we consider a typical scenario of AdaptSFL over edge networks, which consists of three fundamental components:

\begin{itemize}
\item \textbf{Edge devices:} We consider that each client possesses an edge device with computing capability to execute the client-side forward propagation (FP) and back-propagation (BP). The set of participating edge devices is denoted by $\mathcal{N} = \left\{ {1,2,...,N} \right\}$, where $N$ is the number of edge devices. Each edge device $i$ has its local dataset ${\mathcal{D}_i} = \left\{ {{{\bf{x}}_{i,k}},{y_{i,k}}} \right\}_{k = 1}^{{D_i}}$ with ${D_i}$ data samples, where ${{\bf{x}}_{i,k}}$  and ${{{y}}_{i,k}} $ are the $k$-th input data and its corresponding label. The client-side sub-model of edge device $i$ is represented as ${{\bf{w}}_{c,i}}$.

\item \textbf{Edge server:} The edge server is a powerful central computing entity responsible for performing server-side sub-model training. The server-side sub-model of edge device $i$ is denoted by ${{\bf{w}}_{s,i}} = \left[{{\bf{h}}_{s}};{{\bf{h}}_{m,i}}\right]$, where ${\bf{h}}_{s}$ and ${\bf{h}}_{m,i}$ denote the server-side common and server-side non-common sub-models, respectively. The common sub-model means it is shared by all clients and synchronized in every round. The non-common sub-model stems from the fact that some clients have more layers than others on the server, which is aggregated every $I$ rounds together with the client-side sub-models. Furthermore, the edge server also takes charge of collecting important network information, such as the device computing capabilities and channel states, to support optimization decisions.

\item \textbf{Fed server:} The fed server takes charge of client-side sub-model synchronization, periodically aggregating the client-side sub-models from all participating edge devices. For privacy concerns, fed and edge servers usually do not belong to the identical party because a malicious server can recover the raw data if getting both client-side sub-models and smashed data~\cite{pasquini2021unleashing}. 
\end{itemize}

\begin{table}[t]\label{notation}
\caption{{Summary of Important Notations.}}
  \renewcommand{\arraystretch}{1.15}{
  \setlength{\tabcolsep}{1mm}{
\begin{tabular}{ll}
\hline
\textbf{Notation}                                                              & ~~~\textbf{Description}  \\ \hline
~~~$\mathcal{N}$ &  ~~~The set of participating edge devices     \\
~~~${\mathcal{D}_i}$           &  ~~~The local dataset of edge device $i$  \\
~~~$f\left( {\bf{w}} \right)$     &  ~~~The global loss function with parameter ${{\bf{w}}}$ \\
~~~$\mathcal{R}$     &  ~~~The set of training rounds \\
~~~${\mathcal{B}_i}$     &  ~~~The mini-batch drawn from edge device $i$'s local \\
&~~~dataset\\
~~~$b$     &  ~~~The mini-batch size \\
~~~${{\bf{w}}_{c,i}}$/${{\bf{w}}_{s,i}}$            &  ~~~The client-side/server-side sub-model of edge device $i$  \\
~~~${\bf{ h}}_{s}$           &  ~~~The server-side common sub-model  \\
~~~${{\bf{ h}}_{m,i}}$           &  ~~~The server-side non-common sub-model of edge device $i$  \\
 ~~~${{\bf{h}}_{c,i}}$           &  ~~~The forged client-specific model of edge device $i$  \\
~~~$L_c$           &  ~~~The maximum number of layers of client-side sub-models  \\
~~~$L$           &  ~~~The total number of global model layers  \\
~~~${{f _s}}$/${{f _i}}$           &  ~~~The computing capability of server/edge device $i$  \\
~~~${\rho _j}$/${\varpi _j}$           &  ~~~The FP/BP computing workload of propagating the\\
& ~~ first $j$ layer neural  network for one data sample\\
~~~${\psi _j}$/${\chi _j}$           &  ~~~The data size of the activations/activations' gradients\\
&  ~~ at the cut layer $j$\\
~~~${\delta _j}$           &  ~~ The data size of client-side sub-model with the cut layer $j$ \\
~~~${r_i^{U}}$/$r_{i,f}^U$       &  ~~~The uplink transmission rate from edge device $i$ to \\
&  ~~ edge server/fed server.\\
~~~${r_i^{D}}$/$r_{i,f}^D$    &  ~~~The downlink transmission rate from edge server/ \\
&  ~~ fed server to edge device $i$.\\
~~~$r_{s,f}$    &  ~~~The transmission rate from the edge server to fed  \\
&  ~~ server. \\
~~~$r_{f,s}$    &  ~~~The transmission rate from fed server to edge server. \\
~~~$\beta, {\sigma _j^2}, {G _j^2}$     &  ~~ The loss function constants (detailed in Section~\ref{convergence_adaptSFL})\\
~~~$I$, $\mu_{i,j}$           &  ~~~The decision variables (explained in Section~\ref{prob_formu})  \\
~~~$T_1$-$T_6$           &  ~~~The auxiliary variables (to linearize problem~\eqref{problem_1})  \\\hline
\end{tabular}}}
\end{table}

% For edge device $i$, the activation obtained from ${{{\bf{x}}_{i,k}}}$ is represented as ${{\bf{a}}_{i,k}} = \varphi\left( {{{\bf{x}}_{i,k}};{{\bf{w}}_{c,i}}} \right)$, where $\varphi\left( {{{x}};{{w}}} \right)$ maps the relationship between input data ${{x}}$ and its predicted value given model parameter ${{w}}$.  Similarly, based on activation ${{{\bf{a}}_{i,k}}}$, we denote its corresponding predicted value as ${{\hat y}_{i,k}} = \varphi\left( {{{\bf{a}}_{i,k}};{{\bf{w}}_{s,i}}} \right)$. 

% The global model is denoted by ${\bf{w}} = \left[ {{{\bf{w}}_{s,i}};{{\bf{w}}_{c,i}}} \right] $. For edge device $i$, local loss function is represented as  ${f_i}\left( {\bf{w}} \right) = {\mathbb{E}_{{\xi _i} \sim {\mathcal{D}_i}}}[{F_i}\left( {{\bf{w}};{\xi _i}} \right)]$

The global model is denoted by ${\bf{w}} = \left[ {{{\bf{w}}_{s,i}};{{\bf{w}}_{c,i}}} \right] $. The objective of SL is to find the optimal global model ${{\bf{w}}^{\bf{*}}}$ that achieves good performance across all participating devices, which can be formulated as minimizing a finite-sum non-convex global loss function:
\begin{align}\label{minimiaze_loss_function}
\mathop {\min }\limits_{\bf{w}} f\left( {\bf{w}} \right) \buildrel \Delta \over = \mathop {\min }\limits_{\bf{w}} {\frac{{{1}}}{N}} \sum\limits_{i = 1}^N {f_i}({\bf{w}}),
\end{align}
where ${f_i}\left( {\bf{w}} \right)$ represents the local loss function of edge device $i$ over local dataset $\mathcal{D}_i$, i.e., ${f_i}\left( {\bf{w}} \right) \buildrel \Delta \over = {\mathbb{E}_{{\xi _i} \sim {\mathcal{D}_i}}}[{F_i}\left( {{\bf{w}};{\xi _i}} \right)]$, and $\xi _i$ is training randomness from local dataset $\mathcal{D}_i$, {i.e., stochasticity of data order during training caused by the sampling of mini-batch data}. We consider uniformity in the sizes of local datasets across all edge devices for concise expression. Following the standard stochastic optimization setting~\cite{Karimired2018,yu2019linear,karimireddy2020scaffold}, it is assumed that the stochastic gradient is an unbiased estimator of the true gradient, i.e., $\mathbb{E}_{\xi_{i}^{t}\sim \mathcal{D}_{i}}[\nabla F_{i}(\mathbf{w}^{t-1}_{i};\xi_{i}^{t}) \vert \boldsymbol{\xi}^{[t-1]}] = \nabla f_{i}(\mathbf{w}^{t-1}_{i})$, where $\boldsymbol{\xi}^{[t-1]} \overset{\Delta}{=} [\xi_{i}^{\tau}]_{i\in\{1,2,\ldots,N\}, \tau\in\{1,\ldots,t-1\}}$ denotes all the randomness up to training round $t-1$.

To solve problem~\eqref{minimiaze_loss_function}, the conventional SFL framework aggregates both client-side and server-side sub-models every training round. However, aggregating client-side sub-models at each training round leads to communication inefficiency, consuming considerable communication resources, especially for implementing SFL over mobile edge devices. Moreover, the MS has not been studied by jointly considering its impact on training convergence and communication-computing workload. Motivated by this, we propose a state-of-the-art AdaptSFL framework capable of resource-adaptive client-side MA and MS. For readers' convenience, the important notations in this paper are summarized in Table I.

\subsection{The AdaptSFL Framework}\label{subsec_adapt}
This section details the workflow of the proposed AdaptSFL framework. The distinctive feature of AdaptSFL lies in resource-adaptive client-side MA and MS. By optimizing the client-side MA and MS, training latency can be significantly reduced for achieving target learning performance.

Before model training begins, the edge server initializes the ML model, partitions it into client-side and server-side sub-models via layer-wise MS (shown in Section~\ref{solu_appro}), and determines the optimal client-side MA (discussed in Section~\ref{solu_appro}). Afterward, AdaptSFL is executed in $I$ consecutive training rounds. This process loops until the model converges. The training process of AdaptSFL involves two primary stages, i.e.,  split training and client-side MA. The split training is performed in each training round, while the client-side MA occurs every $I$ training round. As depicted in Fig.~\ref{AdaptSFL}, for a training round $t \in \mathcal{R} = \left\{ {1,2,...,R} \right\}$,  the training process of AdaptSFL is detailed as follows.

\textit{a. Split Training Stage:} The split training stage involves model updates for participating edge devices and an edge server every training round. This stage comprises the following five steps.

\textit{a1) Client-side Model Forward Propagation:} In this step, all participating edge devices execute client-side FP in parallel. To be specific, edge device $i$ randomly draws a mini-batch ${\mathcal{B}_i} \subseteq {\mathcal{D}_i}$ with $b$ data samples from its local dataset for model training. The input data and corresponding label of mini-batch in training round $t$ are denoted by ${{\bf{x}}^t_i}$ and ${{\bf{y}}^t_i}$, and ${\bf{w}}^{t-1}_{c,i}$ represents the client-side sub-model of edge device $i$. After feeding a mini-batch of data into the client-side sub-model, activations are generated at the cut layer. The activations of edge device $i$ are represented as
\begin{align}\label{stage_1_1}
{{\bf{a}}^t_i} = \varphi \left( {{\bf{x}}^t_i};{{\bf{w}}^{t-1}_{c,i}} \right), 
\end{align}
where $\varphi\left( {{{x}};{{w}}} \right)$ maps the relationship between input data ${{x}}$ and its predicted value given model parameter ${{w}}$.

\textit{a2) Activations Transmissions:} 
After completing client-side FP, each edge device sends its respective activations and labels to the edge server (usually over wireless channels). The collected activations from participating edge devices are then utilized to fuel server-side model training.

\textit{a3) Server-side Model Forward Propagation and Back-propagation:} After receiving activations from participating edge devices, the edge server feeds these activations into the server-side sub-models to execute the server-side FP. For edge device $i$, the predicted value is expressed as
\begin{align}\label{stage_1_3}
{\bf{\hat y}}^t_i = \varphi\left( {{\bf{a}}^t_i;{{\bf{w}}^{t-1}_{s,i}}} \right), 
\end{align}
where ${{\bf{w}}^{t-1}_{s,i}} = \left[{{\bf{ h}}^{t-1}_{s}};{{\bf{ h}}^{t-1}_{m,i}}\right]$, ${\bf{ h}}^{t-1}_{s}$ and ${\bf{h}}^{t-1}_{m,i}$ denote the server-side common and server-side non-common sub-models, respectively. The predicted value and labels are utilized to calculate the loss function value and further derive the server-side sub-model's gradients.
% \begin{align}\label{stage_5_2}
% {\bf{g}}_{s,i}^t = \left[ {{\nabla _{{{\bf{h}}_s}}}{F_i}({\bf{h}}_s^{t - 1};\xi _i^t);{\nabla _{{{\bf{h}}_m}}}{F_i}({\bf{h}}_{m,i}^{t - 1};\xi _i^t)} \right],
% \end{align}

For the server-side common sub-model, the server updates it\footnote{{The edge server can execute model update for multiple edge devices in either a serial or parallel fashion, which does not affect training performance. Here, we consider the parallel fashion.}} based on
\begin{align}\label{stage_5_2}
{\bf{h}}_s^t = \frac{1}{N}\sum\limits_{i = 1}^N {{\bf{h}}_{s,i}^t},
\end{align}
where ${\bf{h}}_{s,i}^t \leftarrow {\bf{h}}_{s,i}^{t - 1} - \gamma {\nabla _{{{\bf{h}}_s}}}{F_i}({\bf{h}}_{s,i}^{t - 1};\xi _i^t)$ is server-side common sub-model for edge device $i$ and {${\nabla _{\bf{w}}}F({\bf{w}}; \xi)$ represents the gradient of the function $F({\bf{w}}; \xi)$ with respect to the model parameter ${\bf{w}}$.} Since the aggregation in Eqn.~\eqref{stage_5_2} does not incur any communication overhead, we assume it is conducted every training round to improve training convergence. In this way, the server-side common sub-model is equivalent to centralized training (i.e., stochastic gradient descent or SGD).

Due to the heterogeneous cut layers, some edge devices can leave more layers on the server side, resulting in server-side non-common parts of server-side sub-models. Each edge device updates its respective server-side non-common sub-model individually, i.e., 
\begin{align}\label{non_commen_update}
{\bf{h}}_{m,i}^t \leftarrow {\bf{h}}_{m,i}^{t - 1} - \gamma {\nabla _{{{\bf{h}}_m}}}{F_i}({\bf{h}}_{m,i}^{t - 1};\xi _i^t), 
\end{align}
where ${\nabla _{{{\bf{h}}_s}}}{F_i}({\bf{h}}_{s}^{t - 1};\xi _i^t)$ and ${\nabla _{{{\bf{h}}_m}}}{F_i}({\bf{h}}_{m,i}^{t - 1};\xi _i^t)$ represent the stochastic gradients of server-side common and non-common sub-models of edge device $i$, and $\gamma$ is the learning rate.

\textit{a4) Activations' Gradients Transmissions:} After the server-side BP is completed, the edge server transmits the activations' gradients to its corresponding edge devices.

\textit{a5) Client-side Model Back-propagation:} In this step, each edge device updates its client-side sub-model parameters based on the received activations' gradients. For edge device $i$, the client-side sub-model is updated through
\begin{align}\label{client_side_update}
{{\bf{w}}^t_{c,i}} \leftarrow {{\bf{w}}^{t-1}_{c,i}} - {\gamma } {\nabla _{{{\bf{w}}_c}}}{F_i}({\bf{w}}_{c,i}^{t - 1};\xi _i^t).
\end{align}

\textit{b. Client-side Model Aggregation Stage:} The client-side MA stage primarily focuses on aggregating forged client-specific sub-models (including client-side sub-models and server-side non-common sub-models specific to each client) on the fed server, which is executed every $I$ training round. This stage consists of the following three steps.

\textit{b1) Model Uplink Uploading:} In this step, edge devices and edge server simultaneously send the client-side sub-models and server-side non-common sub-models to the fed server over the wireless/wired links.

\textit{b2) Client-side Model Aggregation:} The fed server pairs and assembles the received client-side sub-models and server-side non-common sub-models to forge the client-specific models. Afterwards, these forged client-specific models are aggregated, denoted by
\begin{align}\label{h_c_define}
{\bf{h}}_c^t = \frac{{{1}}}{N} \sum\limits_{i = 1}^N {{\bf{h}}_{c,i}^t} ,
\end{align}
where ${{\bf{h}}^t_{c,i}} = \left[{{\bf{h}}^t_{m,i}};{{\bf{ w}}^t_{c,i}}\right]$ is forged client-specific model of edge device $i$.

\textit{b3) Model Downlink Transmissions:} After completing the client-side MA, the fed server sends updated client-side sub-models and server-side non-common sub-models to corresponding edge devices and the edge server, respectively.

The AdaptSFL training framework is outlined in~\textbf{Algorithm~\ref{AdaptSFL_procedure}}, where $E$ is the total number of stochastic gradients involved in model training.

\begin{algorithm}[t]
	%\textsl{}\setstretch{1.8}
	\renewcommand{\algorithmicrequire}{\textbf{Input:}}
	\renewcommand{\algorithmicensure}{\textbf{Output:}}
	\caption{The AdaptSFL Training Framework}\label{AdaptSFL_procedure}
	\begin{algorithmic}[1]
 \REQUIRE   $b$, $\gamma$, $E$, ${\cal N}$ and $\cal{D}$.
		\ENSURE ${{\bf{w}}^{\bf{*}}}$. 
		\STATE Initialization: ${{\bf{w}}^{{0}}}$, ${{\bf{w}}_i^{{0}}} \leftarrow {{\bf{w}}^{{0}}}$, $I^0 \leftarrow 1$, $\tau \leftarrow 0$ and $\rho  \leftarrow 0$.
          \WHILE{$b\sum\limits_{\eta  = 0}^\tau  {{I^\eta }}  \le E$}
          \STATE Determine $I^\tau$ and ${\boldsymbol{\mu}}^\tau$ based on~\textbf{Algorithm~\ref{BCD-based}}
          \FOR {$t=\rho +1,\rho +2,...,\rho +I^{\tau}$}
          \STATE
          \STATE /** {Runs} {on} {edge} {devices} **/
          \FORALL {edge device ${i \in \,{\cal N}}$ in parallel}
            \STATE  ${{\bf{a}}^t_i} \leftarrow \varphi \left( {{\bf{x}}^t_i};{{\bf{w}}^{t-1}_{c,i}} \right)$
            \STATE Send $\left( {{{\bf{a}}^t_i},{\bf{y}}_i^t} \right)$ to the edge server
          \ENDFOR

	   \STATE
          \STATE /** {Runs} {on} {edge} {server} **/
          \STATE${\bf{\hat y}}^t_i = \varphi\left( {{\bf{a}}^t_i;{{\bf{w}}^{t-1}_{s,i}}} \right)$
          \STATE Calculate loss function value $f\left( {{\bf{w}}^{t - 1}} \right)$
          \STATE ${\bf{h}}_s^t \leftarrow {\bf{h}}_s^{t - 1} - \gamma \frac{1}{N}\sum\limits_{i = 1}^N {{\nabla _{{{\bf{h}}_s}}}{F_i}({\bf{h}}_s^{t - 1};\xi _i^t)} ,$ 
          \STATE ${\bf{h}}_{m,i}^t \leftarrow {\bf{h}}_{m,i}^{t - 1} - \gamma {\nabla _{{{\bf{h}}_m}}}{F_i}({\bf{h}}_{m,i}^{t - 1};\xi _i^t)$
          \STATE Send activations' gradients  to corresponding edge devices

	   \STATE
          \STATE /** {Runs} {on} {edge} {devices} **/
         \FORALL {edge device ${i \in \,{\cal N}}$ in parallel}
           \STATE ${{\bf{w}}^t_{c,i}} \leftarrow {{\bf{w}}^{t-1}_{c,i}} - {\gamma } {\nabla _{{{\bf{w}}_c}}}{F_i}({\bf{w}}_{c,i}^{t - 1};\xi _i^t)$
        \ENDFOR
        \STATE
        \STATE /** {Runs} {on} the {fed} {server} **/
        \IF{$\left( {t - \rho } \right)$ mod $I^\tau$ $=0$}
          \STATE Forge client-side sub-models ${{\bf{h}}^t_{c,i}} = \left[{{\bf{h}}^t_{m,i}};{{\bf{ w}}^t_{c,i}}\right]$
          \STATE ${\bf{h}}_c^t = \frac{{{1}}}{N} \sum\limits_{i = 1}^N {{\bf{h}}_{c,i}^t} $
          \STATE ${\bf{h}}_{c,i}^t \leftarrow {\bf{h}}_c^t$
        \ENDIF
        \ENDFOR
        \STATE
        \STATE Update $\rho  \leftarrow \rho+I^\tau$
        \STATE Update $\tau \leftarrow \tau+1$
        
          \ENDWHILE          
	\end{algorithmic}  
\end{algorithm}

\section{Convergence Analysis of AdaptSFL}\label{convergence_adaptSFL}

In this section, we provide the first convergence analysis of SFL that characterizes the effect of both client-side MA and MS on its training convergence, which serves as the basis for developing an efficient iterative method in Section~\ref{solu_appro}.

Before analysis, we first clarify several notations to simplify the derivation process and results. Considering the variance in MA interval, we define the global model of edge device $i$ as ${\bf{w}}_i^t = \left[ {{\bf{h}}_{s,i}^t;{\bf{h}}_{c,i}^t} \right]$, which contains server-side common sub-model and forged client-specific model. The server-side common sub-models are aggregated each training round, while the client-specific model is aggregated every $I$ training round. The model split point between these two sub-models is $L_c$, defined as the maximum number of layers of client-side sub-models due to discrepancies in model split points among different edge devices. The corresponding gradients of these two sub-models are determined as
\begin{equation}\label{g_ci}
{\bf{g}}_{c,i}^t = \left[ {{\nabla _{{{\bf{h}}_m}}}{F_i}({\bf{h}}_{m,i}^{t - 1};\xi _i^t);{\nabla _{{{\bf{w}}_c}}}{F_i}({\bf{w}}_{c,i}^{t - 1};\xi _i^t)} \right]
\end{equation}
and 
\begin{equation}
{\bf{g}}_{s,i}^t = \sum\limits_{i = 1}^N {{\nabla _{{{\bf{h}}_s}}}{F_i}({\bf{h}}_{s}^{t - 1};\xi _i^t)},
\end{equation}
where ${\bf{h}}_s^t = \frac{{{1}}}{N} \sum\limits_{i = 1}^N {{\bf{h}}_{s,i}^t}$. Therefore, the gradient of the global model is represented as ${\bf{g}}_i^t = \left[ {{\bf{g}}_{s,i}^t; {\bf{g}}_{c,i}^t} \right]$.

In line with the seminal works in distributed stochastic optimization~\cite{zhang2012communication,lian2017can,mania2017perturbed,lin2018don} where convergence analysis is conducted on an aggregated version of individual solutions, this paper focuses on the convergence analysis of ${{\bf{w}}^t} = \frac{{{1}}}{N} \sum\limits_{i = 1}^N {\bf{w}}_i^t$. To analyze the convergence rate of {\bf{Algorithm~\ref{AdaptSFL_procedure}}}, we consider the following two standard assumptions on the loss functions.

\begin{assumption}[\textit{Smoothness}]\label{asp:1}
\textit{Each local loss function ${f_i}\left( {\bf{w}} \right)$ is differentiable and $\beta $-smooth., i.e., for any $\mathbf{w}$ and ${{\bf{w'}}}$, we have}
    \begin{equation}
\left\| {\nabla_{\bf{w}} {f_i}\left( {\bf{w}} \right) - \nabla_{\bf{w}} {f_i}\left( {\bf{w'}} \right)} \right\| \le \beta \left\| {{\bf{w}} - {\bf{w'}}} \right\|,\;\forall i.
    \end{equation}
\end{assumption}

\begin{assumption}[\textit{Bounded variances and second moments}]\label{asp:2}
\textit{The variance and second moments of stochastic gradients for each layer have upper bound:}
    \begin{equation}
\mathbb{E}_{\xi_{i}\sim \mathcal{D}_{i}} \Vert \nabla_{\bf{w}} F_{i}(\mathbf{w}; \xi_{i}) - \nabla_{\bf{w}} f_{i}(\mathbf{w})\Vert^{2} \leq  \sum\limits_{j = 1}^l  {\sigma _j^2}, \forall \mathbf{w}, \;\forall i.
    \end{equation}
        \begin{equation}
\mathbb{E}_{\xi_{i}\sim \mathcal{D}_{i}} \Vert \nabla_{\bf{w}} F_{i}(\mathbf{w}; \xi_{i}) \Vert^{2} \leq \sum\limits_{j = 1}^l  {G _j^2}, \forall \mathbf{w}, \;\forall i ,
    \end{equation}
\textit{where $l$ is number of layers for model  $\mathbf{w}$, ${\sigma _j^2}$ and ${G_j^2}$ denote the bounded variance and second order moments for the $j$-th layer of model $\bf w$.}
\end{assumption}

\begin{lemma} \label{lm:diff-avg-per-node}
Under {\bf{Assumption \ref{asp:1}}, \bf{Algorithm \ref{AdaptSFL_procedure}}} ensures
\begin{align*}
\mathbb{E} [\Vert {\mathbf{h}}_c^{t} - \mathbf{h}^{t}_{c,i}\Vert^{2}] \leq {\mathbbm{1}}_{\{I > 1\}} 4\gamma^{2} I^{2} \sum\limits_{j = 1}^{L_c}  {G _j^2}, \forall i, \forall t,
\end{align*}
\textit{where $I$ denotes client-side model aggregation interval and ${\mathbf{h}}_c^{t}$ is defined in {Eqn. \eqref{h_c_define}}.}
\end{lemma}

\begin{proof}
See Appendix~\ref{aa}.
\end{proof}

\vspace{0.15cm}
\begin{theorem}\label{theorem1}
Under {\bf{Assumption \ref{asp:1}}} and {\bf{Assumption \ref{asp:2}}}, if $0 < \gamma \leq \frac{1}{\beta}$ in {\bf{Algorithm \ref{AdaptSFL_procedure}}}, then for all $R\geq 1$, we have 
{ \footnotesize \begin{equation}\label{convergence_bound}
\begin{aligned}
\frac{1}{R}\sum\limits_{t = 1}^R \mathbb{E}  [{\Vert\nabla _{\bf{w}}}f({{\bf{w}}^{t - 1}})\Vert{^2}] \le \frac{2\vartheta}{{\gamma R}} \!+\! \frac{{\beta \gamma \sum\limits_{j = 1}^L {\sigma _j^2} }}{N}
 + {\mathbbm{1}}_{\{I > 1\}} 4{\beta ^2}{\gamma ^2}{I^2}\sum\limits_{j = 1}^{{L_c}}\! {G_j^2} ,
\end{aligned}
% \frac{1}{T} \sum_{t=1}^{T} \mathbb{E} [\Vert \nabla_{\bf{w}} f({\mathbf{w}}^{t-1})\Vert^{2}] \leq \frac{2}{\gamma T} \left(f({\mathbf{w}}^{0}) -f^{\ast}\right) +\frac{\beta\gamma \sum\limits_{j = 1}^{L_c}  {G _j^2} }{N}  \nonumber + 4\beta^2\gamma^{2} I^{2} L_c G^{2}
\end{equation}}%
where $\vartheta  = f({{\bf{w}}^0}) - {f^ * }$, $L$ and $f^{\ast}$ represent the total number of global model layers and minimum value of problem~\eqref{minimiaze_loss_function}. 
\end{theorem}

\begin{proof}
See Appendix B.
\end{proof}
\rev{{\bf{Remark:}} Compared to conventional FL~\cite{mcmahan2017communication,wang2019adaptive,wu2023faster}, the convergence bound of AdaptSFL exhibits two key differences. First, the varied client-side and server-side MA intervals introduce an additional error term proportional to $I^2$, as client-side updates are aggregated every $I$ rounds, causing model staleness and slowing model convergence. Second, the MS decision variable $L_c$ governs the synchronization efficiency: a smaller $L_c$ increases the proportion of the model updated each round, improving convergence speed, whereas a larger $L_c$ degrades convergence performance.}

% For Theorem 1, we can derive the following corollary.
Substituting Eqn.~\eqref{convergence_bound} into Eqn.~\eqref{accuracy_cons_corollary} yields {\bf{Corollary 1}}, revealing a lower bound on the number of training rounds for achieving target convergence accuracy.
\begin{corollary}\label{theorem1}
The number of {training rounds $R$} for achieving target convergence accuracy $\varepsilon$, i.e., satisfying
{ \begin{equation}\label{accuracy_cons_corollary}
\frac{1}{R} \sum_{t=1}^{R} \mathbb{E} [\Vert \nabla_{\bf{w}} f({\mathbf{w}}^{t-1})\Vert^{2}] \le \varepsilon,
\end{equation}}
is given by
{ \begin{equation}\label{lowest_com_num}
R  \ge \frac{2 \vartheta}{{\gamma \bigg( {\varepsilon  - \frac{{\beta \gamma \sum\limits_{j = 1}^L {\sigma _j^2}}}{N} - {\mathbbm{1}}_{\{I > 1\}} 4{\beta ^2}{\gamma ^2}{I^2}\sum\limits_{j = 1}^{{L_c}} {G_j^2}} \bigg)}}.
\end{equation}}
\end{corollary}
{\textbf{Insights:}  Eqn.~\eqref{lowest_com_num} shows that the number of {training rounds} for achieving target convergence accuracy $\varepsilon$ decreases as $I$ decreases and the cut layer becomes shallow, indicating a faster model convergence rate. Similarly, for a given {training round $R$}, more frequent client-side MA and a shallower cut layer result in higher converged accuracy (i.e., smaller $\varepsilon$). These observations align with the experimental results illustrated in Fig.~\ref{fig:motivation_1} and Fig.~\ref{fig:motivation_2}. However, the improved training performance comes at the cost of increasing communication overhead, leading to longer latency per training round. Therefore, it is essential to optimize MS and client-side MA to accelerate SFL under communication and computing constraints.}

\section{Problem Formulation}\label{prob_formu}

The convergence analysis results quantify the impact of client-side MA and MS on model convergence and communication efficiency. In this section, guided by the convergence bounds in Section~\ref{convergence_adaptSFL}, we formulate a joint client-side MA and MS optimization problem. The objective is to minimize the training latency of AdaptSFL over edge networks with heterogeneous participating devices. Subsequently, we propose a resource-adaptive client-side MA and MS strategy in Section~\ref{solu_appro}. For clarity, the decision variables and definitions are listed below.

\begin{itemize}
\item $I$: $I \in {\mathbb{N}^ + }$ denotes the client-side MA decision variable, which indicates that the forged client-side sub-model is aggregated every $I$ training rounds.
\item ${\boldsymbol\mu}$: $\mu_{i,j} \in \left\{ {0,1} \right\}$ is the MS decision variable, where $\mu_{i,j}=1$ indicates that the $j$-th neural network layer is selected as the cut layer for edge device $i$, and 0 otherwise.
${\boldsymbol\mu}  = \left[ {{\mu _{1,1}},{\mu _{1,2}},...,{\mu _{N,L}}} \right]$ represents the collection of MS decisions.
\end{itemize}

\subsection{Training Latency Analysis}
In this section, we provide the training latency analysis for each step of AdaptSFL. Without loss of generality, we focus on one training round for analysis. The training round number index $t$ is omitted for notational simplicity. To begin, we analyze the latency of split training stage for one training round as follows\footnote{{Consistent with prior studies~\cite{yang2020energy,wu2023split,wang2021hivemind,chen2022federated,lin2023efficient}, we ignore the cost of network and computing capacity probing in our analysis, as these costs are negligible compared to the excessive communication and computation costs of the training process. Moreover, lightweight tools such as iperf~\cite{tirumala1999iperf} for bandwidth measurement and Linux utilities like top for monitoring computing load are highly efficient, enabling
efficient probing with minimal overhead. }}.

\textit{a1) Client-side Model Forward Propagation Latency:}
In this step, each edge device utilizes its local dataset to conduct the client-side FP. The computing workload (in FLOPs\footnote{\rev{Consistent with prior distributed machine learning literature~\cite{zhang2022scalable,deng2022low,lin2023efficient,wu2023split}, we adopt FLOPs as a standardized, hardware-independent metric to quantify computing cost, thus enabling fair comparisons across different model architectures and simulation settings.}
}) of the client-side FP for edge device $i$ per data sample is denoted by $\Phi_{c,i}^{F}\! \left( \boldsymbol{\mu}  \right) = \sum\limits_{j = 1}^{L} {{\mu _{i,j}}} {\rho _j}$, where ${\rho _j}$ is the FP computing workload of propagating the first $j$ layer of neural network for one data sample. Considering that each edge device randomly draws a mini-batch containing $b$ data samples for executing the client-side FP, the client-side FP latency of edge device $i$ is calculated as 
\begin{align}\label{client_FP_latency}
T_i^{F} = \frac{{b\,\Phi_{c,i}^{F}\! \left( \boldsymbol{\mu}  \right)}}{{{f_i}}} ,\;\forall i \in \mathcal{N},
\end{align}
where ${{f _i}}$ denotes the computing capability of edge device $i$, i.e., the number of float-point operations per second (FLOPS).

\textit{ a2) Activations Transmissions Latency:}
After completing the client-side FP, each edge device sends the activations generated at the cut layer to the edge server. Let $\Gamma_{a,i} \left( \boldsymbol{\mu}  \right) = \sum\limits_{j = 1}^{L} {{\mu _{i,j}}} {\psi _j}$ represent the data size (in bits) of activations for edge device $i$, where ${\psi _j}$ denotes the data size of activations at the cut layer $j$. Since each edge device employs a mini-batch with $b$ data samples, the activations transmissions latency of edge device $i$ is determined as
\begin{align}\label{smashed_trans_latency}
T_{a,i}^{U} = \frac{{b\Gamma_{a,i} \left( \boldsymbol{\mu}  \right)}}{{r_i^{U}}},\;\forall i \in \mathcal{N}.
\end{align}
where ${r_i^{U}}$ is uplink transmission rate from edge device $i$ to the edge server.

\textit{a3) Server-side Model Forward Propagation and Back-propagation Latency:}
This step involves executing the server-side FP and BP with collected activations from all participating edge devices. Let $\Phi _s^F\left( {\boldsymbol{\mu }} \right) = \sum\limits_{i = 1}^N {\sum\limits_{j = 1}^L {{\mu _{i,j}}} \left( {{\rho _L} - {\rho _j}} \right)} $ represent the computing workload of the server-side FP for each data sample. Similarly, the computing workload of the server-side BP per data sample is denoted by $\Phi _s^{B}\left( \boldsymbol{\mu}  \right) = \sum\limits_{i = 1}^N {\sum\limits_{j = 1}^{L} {{\mu _{i,j}}} \left( {{\varpi_L} - {\varpi _j}} \right)}$, where ${\varpi _j}$ is the BP computing workload of propagating the first $j$ layer of neural network for one data sample. Consequently, the server-side model FP and BP latency can be obtained from
\begin{align}\label{server_FP_latency}
T_s^{F} = \frac{{b\,\Phi _s^{F}\left( \boldsymbol{\mu}  \right)}}{{{f_s}}}
\end{align}
and 
\begin{align}\label{server_BP_latency}
T_s^{B} = \frac{{b\,\Phi _s^{B}\left( \boldsymbol{\mu}  \right)}}{{{f_s}}},
\end{align}
where ${{f _s}}$ denotes the computing capability of the edge server.

\textit{a4) Activations' Gradients Transmissions Latency:}
After the server-side BP is completed, the edge server sends the activations' gradients to the corresponding edge devices. Let $\Gamma_{g,i} \left( \boldsymbol{\mu}  \right) = \sum\limits_{j = 1}^{L} {{\mu _{i,j}}} {\chi _j}$ represent the data size of activations' gradients for edge device $i$, where ${\chi _j}$ denotes the data size of activations' gradients at the cut layer $j$. Therefore, the activations' gradients transmissions latency of edge device $i$  is expressed as
\begin{align}\label{downlink_latency}
T_{g,i}^D = \frac{{ {b } {\Gamma _{g,i}}\left( {\boldsymbol{\mu} } \right)}}{{r_i^D}},\;\forall i \in {\cal N}.
\end{align}
where ${r_i^{D}}$ is downlink transmission rate from edge server to edge device $i$.

\textit{a5) Client-side Model Back-propagation Latency:}
In this step, each edge device executes client-side BP according to the received activations' gradients. Let $\Phi _{c,i}^{B}\left( \boldsymbol{\mu}  \right) = \sum\limits_{j = 1}^{L} {{\mu _{i,j}}} {\varpi _j}$ denote the computing workload of the client-side BP for edge device $i$ per data sample. As a result, the client-side BP latency of edge device $i$ can be obtained by
\begin{align}\label{client_BP_latency}
T_i^{B} = \frac{{b\, \Phi _{c,i}^{B}\left( \boldsymbol{\mu}  \right)}}{{{f_i}}},\;\forall i \in \mathcal{N}.
\end{align}

Then, we analyze the latency of client-side MA stage every $I$ training round as follows.

\textit{b1) Model Uplink Transmissions Latency:} In this step, each edge device transmits its client-side sub-model to the fed server, while edge server uploads the server-side non-common sub-models to the fed server for client-side MA. Let $\Lambda_{c,i} \left( \boldsymbol{\mu}  \right) = \sum\limits_{j = 1}^{L} {{\mu _{i,j}}} {\delta _j}$ represent the data size of client-side sub-model for edge device $i$, where ${\delta _j}$ is the data size of client-side sub-model with the cut layer $j$. The total data size of the exchanged server-side non-common sub-models is denoted by ${\Lambda _{s}}\left( {\boldsymbol{\mu }} \right) = N \mathop {\max }\limits_i \left\{ { \sum\limits_{j = 1}^L {{\mu _{i,j}}} {\delta _j}} \right\} - \sum\limits_{i = 1}^N \sum\limits_{j = 1}^L {{\mu _{i,j}}} {\delta _j}$.  Therefore, the uplink transmissions latency for the client-side sub-model and server-side non-common sub-models are expressed as 
\begin{align}\label{smashed_trans_latency}
T_{c,i}^U = \frac{{{\Lambda _{c,i}}\left( {\boldsymbol{\mu }} \right)}}{{r_{i,f}^U}}, \;\forall i \in \mathcal{N}
\end{align}
and 
\begin{align}\label{smashed_trans_latency}
T_s^U = \frac{{{\Lambda _{s}}\left( {\boldsymbol{\mu }} \right)}}{{r_{s,f}}},
\end{align}
where $r_{i,f}^U$ and $r_{s,f}$ represent the transmissions rate from edge device $i$ and edge server to the fed server, respectively.

\textit{b2) Client-side Model Aggregation:} The fed server aggregates the client-side sub-models and server-side non-common sub-models to forge the aggregated client-specific model. For simplicity, the latency for this part is ignored since the aggregation delay is typically much smaller than other steps~\cite{shi2020joint,xia2021federated}.

\textit{b3) Model Downlink Transmissions Latency:} After completing client-side MA, the fed server sends the updated client-side sub-model and server-side non-common sub-model to the corresponding edge devices and edge server. Similarly, the downlink transmission latency for the client-side sub-model and server-side non-common sub-models are calculated by 
\begin{align}\label{smashed_trans_latency}
T_{c,i}^D = \frac{{{\Lambda _{c,i}}\left( {\boldsymbol{\mu }} \right)}}{{r_{i,f}^D}}, \;\forall i \in \mathcal{N}
\end{align}
and 
\begin{align}\label{smashed_trans_latency}
T_s^D = \frac{{{\Lambda _{s}}\left( {\boldsymbol{\mu }} \right)}}{{r_{f,s}}},
\end{align}
where $r_{i,f}^D$ and $r_{f,s}$ denote the transmission rate from the fed server to edge device $i$ and edge server, respectively.

\subsection{Joint Client-side Model Aggregation and Model Splitting Problem Formulation}

\begin{figure}[t]
\vspace{-.5ex}
\setlength\abovecaptionskip{3pt}
\centering
\subfigure[Split training.]{
    \includegraphics[height=2.7cm,width=4cm]{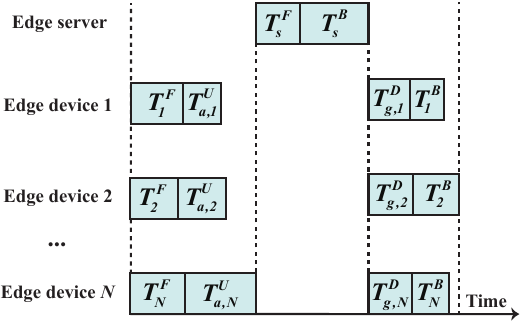}
    \label{sfig:split_training}
}
% \hspace{.1cm}
\subfigure[Client-side MA.]{
    \includegraphics[height=2.7cm,width=2.8cm]{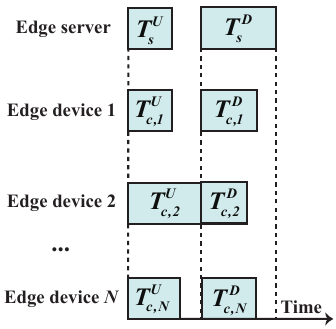}
    \label{sfig:client_side_aggrega}
}
    \caption{An illustration of split training and client-side MA stages.}
    \label{fig:training_2_stage}
    \vspace{-2ex}
\end{figure}

This section presents the problem formulation of client-side MA and MS with the objective of minimizing the time consumed for model convergence over edge networks. As illustrated in Fig.~\ref{fig:training_2_stage}, considering that split training is executed per training round and client-side MA occurs every $I$ training round, the total training latency 
for $I$ training rounds can be derived as
{{ \begin{align}\label{total_latency}
T\!\left( {I,\boldsymbol{\mu} } \right) &\!=\!I\!\left\{ \!{\mathop {\max }\limits_i\! \left\{ {T_i^F \!+\! T_{a,i}^U} \right\} \!\!+\! T_s^F \!+\! T_s^B \!+\! \mathop {\max }\limits_i \left\{ {T_{g,i}^D + T_i^B} \right\}} \!\right\} \nonumber\\&+\mathop {\max }\limits_i \left\{ {T_{c,i}^U,T_s^U} \right\} + \mathop {\max }\limits_i \left\{ {T_{c,i}^D,T_s^D} \right\}.
\end{align}}}

Recalling the derived convergence bound in~\eqref{convergence_bound},  client-side MA interval balances the tradeoffs between communication overhead and training convergence, while MS significantly impacts communication-computing overhead, as well as model convergence. Hence, jointly optimizing the client-side MA interval and MS is imperative to expedite the training process. Observing these facts, we formulate the following optimization problem to minimize the time consumed for model convergence:
\begin{align}\label{time_minimize_problem}
\mathcal{P}:&\mathop {{\rm{min}}}\limits_{{I},{\boldsymbol{\mu}}}  MT({I},{\boldsymbol{\mu }})   \\
&\mathrm{s.t.} ~\mathrm{C1:}~\frac{1}{R}\sum\limits_{t = 1}^R \mathbb{E} [{\Vert\nabla _{\bf{w}}}f({{\bf{w}}^{t - 1}})\Vert{^2}] \le \varepsilon, \nonumber \\
&~\mathrm{C2:}~{\mu _{i,j}} \in \left\{ {0,1} \right\}, \quad\forall i \in \mathcal{N}, \; j = 1,2,...,L, \nonumber\\
&~\mathrm{C3:}~\sum\limits_{j = 1}^{L} {{\mu _{i,j}}}  = 1, \quad \forall i \in \mathcal{N}, \nonumber \\
&~\mathrm{C4:}~I \in {\mathbb{N}^ + }. \nonumber
\end{align}
where {$M = \frac{R}{I}$} denotes the number of rounds of model aggregation for achieving model convergence. Constrain $\mathrm C1$ guarantees global convergence accuracy; $\mathrm C2$ and $\mathrm C3$ ensure that each edge device has a single cut layer, so that the global model is partitioned into client-side and server-side sub-models; $\mathrm C4$ denotes that the client-side MA decision variable is a positive integer.

Problem~\eqref{time_minimize_problem} is a combinatorial problem with a non-convex mixed-integer non-linear objective function, which is NP-hard in general. In this case, achieving an optimal solution via polynomial-time algorithms is infeasible.

\section{Solution Approach}\label{solu_appro}

In this section, we present an efficient iterative algorithm by decoupling the problem~\eqref{time_minimize_problem} into MA and MS subproblems. For each subproblem, the optimal solution can be obtained. We first utilize {\bf{Corollary 1}} to explicitly express $M$. It is clear that $M$ is proportional to the objective function, and thus the objective function is minimized if and only if inequality~\eqref{lowest_com_num} holds as an equality. Thus, the original problem~\eqref{time_minimize_problem} can be converted into 
\begin{align}\label{problem_1}
\mathcal{P'}:&\mathop {{\rm{min}}}\limits_{{I},{\boldsymbol{\mu}}} \Theta ( I, {\boldsymbol{\mu }} )   \\
&\mathrm{s.t.} ~\mathrm{C2}-\mathrm{C4}, \nonumber
\end{align}
where 
\begin{align}
\Theta ( I, {\boldsymbol{\mu }} )= {\frac{{2\vartheta }}{{\!\gamma I \bigg( {\!\varepsilon \! - \!\frac{{\beta \gamma \sum\limits_{j = 1}^L {\sigma _j^2} }}{N} \!- 4{\beta ^2}{\gamma ^2}{I^2}\sum\limits_{j = 1}^{{L_c}} {G_j^2} } \bigg)}}}T({I},{\boldsymbol{\mu }}).\nonumber
\end{align}

Problem~\eqref{problem_1} is still a mixed-integer non-linear programming, and the non-convexity of the objective function makes it still intractable. To solve this problem, we introduce a set of constants ${\bf{\widetilde G}} = \left[ {{{\widetilde G}^2_1},{{\widetilde G}^2_2},...,{{\widetilde G}^2_L}} \right]$, where ${\widetilde G}^2_j$ denotes the cumulative sum of the bounded second order moments for the first $j$ layers of neural network, defined as
\begin{align}\label{substitution_variable}
{{\widetilde G}^2_j} = \sum\limits_{k = 1}^j {G_k^2}. 
\end{align}

Therefore, $\sum\limits_{j = 1}^{{L_c}} {G_j^2}$ in the objective function can be reformulated as $\mathop {\max }\limits_i \left\{ {\sum\limits_{j = 1}^L {{\mu _{i,j}}{\widetilde G}_j^2} } \right\} $. Additionally, we introduce a set of auxiliary variables ${\bf{T}} = \left[ {{T_1},{T_2},...,{T_6}} \right]$ to linearize the objective function, i.e., $\mathop {\max }\limits_i \big \{ {\sum\limits_{j = 1}^L {{\mu _{i,j}}\widetilde G_j^2} } \big\} \le {T_1}$, $\mathop {\max }\limits_i \big\{ {\sum\limits_{j = 1}^L {{\mu _{i,j}}} {\delta _j}} \big\} \le {T_2}$, $\mathop {\max }\limits_i \left\{ {T_i^F + T_{a,i}^U} \right\} \le {T_3}$ ,$\mathop {\max }\limits_i \left\{ {T_{g,i}^D + T_i^B} \right\} \le {T_4}$, $\mathop {\max }\limits_i \left\{ {T_{c,i}^U,T_s^U} \right\} \le {T_5}$, and $\mathop {\max }\limits_i \left\{ {T_{c,i}^D,T_s^D} \right\} \le {T_6}$. Problem~\eqref{problem_1} can therefore be transformed into
{ \begin{align}\label{problem_2}
\mathcal{P''}:&\mathop {{\rm{min}}}\limits_{{I}, {\boldsymbol{\mu}}, \bf{T}, } \Theta' ( I, {\boldsymbol{\mu }}, \bf{T} )   \\
&\mathrm{s.t.} ~\mathrm{C2}-\mathrm{C4}, \nonumber\\
&~\mathrm{R1:}~\sum\limits_{j = 1}^L {{\mu _{i,j}}\widetilde G_j^2}  \le {T_1}, \quad\forall i \in \mathcal{N}, \nonumber\\
&~\mathrm{R2:}~\sum\limits_{j = 1}^L {{\mu _{i,j}}} {\delta _j} \le {T_2}, \quad\forall i \in \mathcal{N}, \nonumber\\
&~\mathrm{R3:}~\frac{{b{\mkern 1mu} \sum\limits_{j = 1}^L {{\mu _{i,j}}} {\rho _j}}}{{{f_i}}} + \frac{{b\sum\limits_{j = 1}^L {{\mu _{i,j}}} {\psi _j}}}{{r_i^U}} \le {T_3}, \quad\forall i \in \mathcal{N}, \nonumber\\
&~\mathrm{R4:}~\frac{{b\sum\limits_{j = 1}^L {{\mu _{i,j}}} {\chi _j}}}{{r_i^D}} + \frac{{b{\mkern 1mu} \sum\limits_{j = 1}^L {{\mu _{i,j}}} {\varpi _j}}}{{{f_i}}} \le {T_4}, \quad\forall i \in \mathcal{N},  \nonumber\\
&~\mathrm{R5:}~\frac{{\sum\limits_{j = 1}^L {{\mu _{i,j}}} {\delta _j}}}{{r_{i,f}^U}} \le {T_5}, \quad\forall i \in \mathcal{N},  \nonumber\\
&~\mathrm{R6:}~\frac{{N{T_2} - \sum\limits_{i = 1}^N {\sum\limits_{j = 1}^L {{\mu _{i,j}}} } {\delta _j}}}{{r_{s,f}}} \le {T_5}, \nonumber\\
&~\mathrm{R7:}~\frac{{\sum\limits_{j = 1}^L {{\mu _{i,j}}} {\delta _j}}}{{r_{i,f}^D}} \le {T_6} \quad\forall i \in \mathcal{N},  \nonumber\\
&~\mathrm{R8:}~\frac{{N{T_2} - \sum\limits_{i = 1}^N {\sum\limits_{j = 1}^L {{\mu _{i,j}}} } {\delta _j}}}{{r_{f,s}}} \le {T_6},  \nonumber
\end{align}}
where
{ \begin{align}
\Theta' ( I, {\boldsymbol{\mu }}, {\bf{T}} ) =\frac{{2\vartheta \left\{ {I\left\{ {{T_3} + T_s^F + T_s^B + {T_4}} \right\} + {T_5} +\! {T_6}} \right\}}}{{\gamma I\bigg( {\varepsilon  - \frac{{\beta \gamma \sum\limits_{j = 1}^L {\sigma _j^2} }}{N} - 4{\beta ^2}{\gamma ^2}{I^2}{T_1}} \bigg)}}. 
\end{align}}

From problem~\eqref{problem_2}, we observe that the introduced auxiliary variables are tightly coupled with the original decision variables, which makes it difficult to solve the problem directly. Hence, we decompose the original problem~\eqref{problem_2} into two less complicated subproblems based on different decision variables and develop efficient algorithms to tackle each subproblem.

We fix the variables $\boldsymbol{\mu}$ and $\bf{T}$ to investigate the subproblem involving client-side MA interval, which is expressed as
\begin{align}\label{subproblem_1}
\mathcal{P}_1:&\mathop {{\rm{min}}}\limits_{{I}} \Theta' ( I )  \\
&\mathrm{s.t.} ~\mathrm{C4}. \nonumber
\end{align}
Then, we have the following theorem:
\begin{theorem}\label{theorem2}
The optimal client-side MA interval is given by
\begin{equation}\label{accuracy_cons}
{I^*} = \left\{ {\begin{array}{*{20}{c}}
1&{I' \le 1}\\
{\arg {{\min }_{I \in \left\{ {\left\lfloor {I'} \right\rfloor ,\left\lceil {I'} \right\rceil } \right\}}}\Theta' \left( I \right)}&{I' > 1}
\end{array}}, \right.
\end{equation}
where $I'$ with $\Xi \left( I' \right) = 0$ can be easily obtained by Newton-Raphson method and $\Xi \left( I \right)$ is defined in~\eqref{E_I}.
\end{theorem}

\begin{proof}
See Appendix C.
\end{proof}

Therefore, we obtain an optimal solution for problem~\eqref{subproblem_1}.

\begin{algorithm}[t]
	%\textsl{}\setstretch{1.8}
	\renewcommand{\algorithmicrequire}{\textbf{Input:}}
	\renewcommand{\algorithmicensure}{\textbf{Output:}}
	\caption{Dinkelbach Method for Solving Problem~\eqref{subproblem_2}}\label{dinkelback_method}
	\begin{algorithmic}[1]
        \REQUIRE  $\lambda^{(0)}$ satisfying $\Upsilon ( {\lambda^{(0)}, {\boldsymbol{\mu }}^{(0)},{\bf{T}}^{(0)}} ) \!\ge\! 0$, ${\boldsymbol{\mu }}^{(0)}$, ${\bf{T}}^{(0)}$ and ${\varepsilon _d}$.
		\ENSURE ${{\boldsymbol{\mu }}^{\bf{*}}}$ and $\lambda^*$. 
          \STATE Initialization:  $n  \leftarrow 1$.
          \REPEAT     
          \STATE Solve problem~\eqref{subproblem_2_linear} with $\lambda=\lambda^{(n)}$ to obtain ${\boldsymbol{\mu }}^{(n)}$ and ${\bf{T}}^{(n)}$ 
          \STATE Calculate $\Upsilon ( {\lambda^{(0)},{\boldsymbol{\mu }}^{(n)},{\bf{T}}^{(n)}} ) $      
          \STATE ${\lambda ^{(n+1)}} \leftarrow \frac{{Q\left( {{\boldsymbol{\mu }}^{(n)},{{\bf{T}}^{(n)}}} \right)}}{{P\left( {\boldsymbol{\mu }}^{(n)},{{\bf{T}}^{(n)}} \right)}}$,
          \STATE $n  \leftarrow n+1$
         \UNTIL{$|\Upsilon ( {\lambda^{(n)},{\boldsymbol{\mu }}^{(n)},{\bf T}^{(n)}} )|\le {\varepsilon _d}$}    
	\end{algorithmic}  
\end{algorithm}

By fixing the decision variable $I$, we convert problem~\eqref{problem_2} into a standard mixed-integer linear fractional  programming (MILFP) with respect to $\boldsymbol{\mu}$ and ${\bf{T}}$, which is given by
\begin{align}\label{subproblem_2}
\mathcal{P}_2:&\mathop {{\rm{min}}}\limits_{{\boldsymbol{\mu}, {\bf{T}}}} \Theta' ( {\boldsymbol{\mu }}, {\bf{T}} )   \\
&\mathrm{s.t.} ~\mathrm{C2}, ~\mathrm{C3},~\mathrm{R1}-\mathrm{R8}. \nonumber
\end{align}

We utilize the Dinkelbach algorithm~\cite{dinkelbach1967nonlinear} to efficiently solve problem~\eqref{subproblem_2} by reformulating it as mixed-integer linear programming, which achieves the optimal solution~\cite{yue2013reformulation,rodenas1999extensions}. Introducing the fractional parameter $\lambda$, we reformulate problem~\eqref{subproblem_2} as
\begin{align}\label{subproblem_2_linear}
\mathcal{P}'_2:&\mathop {{\rm{min}}}\limits_{{\boldsymbol{\mu}, {\bf{T}}}}\Upsilon ( {\lambda, {\boldsymbol{\mu }},{\bf{T}}} ) = Q( {\boldsymbol{\mu }},{\bf{T}} ) -\lambda P ( {\boldsymbol{\mu }}, {\bf{T}} )   \\
&\mathrm{s.t.} ~\mathrm{C2}, ~\mathrm{C3},~\mathrm{R1}-\mathrm{R8}, \nonumber
\end{align}
where 
{\small\begin{align}
&Q\!\left( {{\boldsymbol{\mu }}, {\bf{T}}} \right) = 2\vartheta \left\{ {I\left\{ {{T_3}\! +\! T_s^F\! + \!T_s^B \!+ \!{T_4}} \right\}\! + \!{T_5}\! +\! {T_6}\!} \right\}, \nonumber\\
&P\!\left( {{\boldsymbol{\mu }}, {\bf{T}}} \right) = {{\gamma I\Bigg( {\varepsilon  - \frac{{\beta \gamma \sum\limits_{j = 1}^L {\sigma _j^2} }}{N} - 4{\beta ^2}{\gamma ^2}{I^2}{T_1}} \Bigg)}}. \nonumber
\end{align}}

The Dinkelbach algorithm iteratively solves problem~\eqref{subproblem_2_linear}, which is described as follows: We first initialize the fractional parameter $\lambda^{(0)}$ such that $\Upsilon ( {\lambda^{(0)}, {\boldsymbol{\mu }}^{(0)},{{\bf{T}}}^{(0)}} ) \!\ge\! 0$. At each iteration, we obtain the optimal solution by solving problem~\eqref{subproblem_2_linear} with $\lambda^{(n)}$. After that, we update the fractional parameter $\lambda^{(n+1)}$. We repeat the iterations until $|\Upsilon ( {\lambda^{(n)},{\boldsymbol{\mu }}^{(n)},{\bf{T}}^{(n)}} )| < {\varepsilon _d}$, where ${\varepsilon _d}$ is the convergence tolerance. The pseudo-code of the Dinkelbach algorithm for solving problem~\eqref{subproblem_2} is given in \textbf{Algorithm~\ref{dinkelback_method}}.

As aforementioned, we decompose the original problem~\eqref{time_minimize_problem} into two tractable subproblems $\mathcal{P}1$ and $\mathcal{P}2$ according to different decision variables and develop efficient algorithms to solve
each subproblem optimally. Based on this, we propose an iterative block-coordinate descent (BCD)-based algorithm~\cite{tseng2001convergence} to solve problem~\eqref{time_minimize_problem}, which is shown in~\textbf{Algorithm~\ref{BCD-based}}. \rev{The key parameters required for executing algorithm (e.g., $\beta$, ${G _j^2}$ and ${\sigma _j^2}$) are estimated following the approach in~\cite{wang2019adaptive} and lightweight tools can be utilized to measure network resources, such as iperf~\cite{tirumala1999iperf} for bandwidth estimation and system monitoring tools (e.g., Linux top or Android performance APIs) for computing speed.} It is noted that, given dynamic network and training conditions, \textbf{Algorithm~\ref{BCD-based}} can be conducted every $I$ rounds or longer to accelerate SFL under time-varying edge environments, as illustrated in \textbf{Algorithm~\ref{AdaptSFL_procedure}}.

\begin{algorithm}[t]
	%\textsl{}\setstretch{1.8}
	\renewcommand{\algorithmicrequire}{\textbf{Input:}}
	\renewcommand{\algorithmicensure}{\textbf{Output:}}
	\caption{BCD-based Algorithm.}\label{BCD-based}
	\begin{algorithmic}[1]
        \REQUIRE  ${I}^{(0)}$, ${\boldsymbol{\mu }}^{(0)}$, ${{\bf{T}}}^{(0)}$ and ${\varepsilon _b}$.
		\ENSURE  ${I^{\bf{*}}}$ and ${{\boldsymbol{\mu }}^{\bf{*}}}$  
          \STATE Initialization:  $\tau  \leftarrow 0$.
          \REPEAT 
          \STATE$\tau  \leftarrow \tau+1$
           \STATE Update $I^{(\tau)}$ based on \textbf{Theorem 2}
           \STATE Update ${\boldsymbol{\mu }}^{(\tau)}$ and  ${{\bf{T}}}^{(\tau)}$ based on \textbf{Algorithm 2} 
         \UNTIL{\small{$|\Theta' ( I^{(\tau)}, {{\boldsymbol{\mu }}^{(\tau)}}, {{\bf{T}}^{(\tau)}} ) - \Theta' ( I^{(\tau-1)}, {\boldsymbol{\mu }}^{(\tau-1)}, {{\bf{T}}^{(\tau-1)}}| \!\le \!{\varepsilon _b}$}}    
         
	\end{algorithmic}  
\end{algorithm}

\section{Performance Evaluation}\label{simu_results}
This section provides numerical results to evaluate the learning performance of the proposed AdaptSFL framework and the effectiveness of tailored client-side MA and MS strategy.

\subsection{Simulation Setup}\label{simu_setup}
We implement AdaptSFL using Python 3.7 and PyTorch 1.9.1., and train it on a ThinkPad P17 Gen1 laptop equipped with an NVIDIA Quadro RTX 3000 GPU, Intel i9-10885H CPUs and 4TB SSD. In our experiments, we deploy $N$ edge devices and $N$ is set to 20 by default unless specified otherwise. The computing capability of each edge device is uniformly distributed within $[1, 2]$ TFLOPS, and the computing capability of the server is set to $20$ TFLOPS. The uplink transmission rates from edge device $i$ to the edge server and fed server follow uniform distribution wit hin $[75, 80]$ Mbps, and the corresponding downlink rates are set to $370$ Mbps. \rev{It is noted that the computing capability of each edge device and its uplink transmission rates to both edge server and fed server vary across every training round, reflecting dynamic variation in network resources throughout the training process.} For convenience, the inter-server transmission rate, namely ${r_s^{U}}$ and ${r_s^{D}}$, are identically set to $400$ Mbps. The mini-batch size and learning rate are set to 16 and $5\times {10^{-4}}$. For readers' convenience, the detailed simulation parameters are summarized in Table II.

\begin{figure}[t]
\vspace{-.5ex}
\setlength\abovecaptionskip{3pt}
\centering
\subfigure[\rev{CIFAR-10 on VGG-16 under IID setting.}]{
\includegraphics[width=.44\columnwidth]{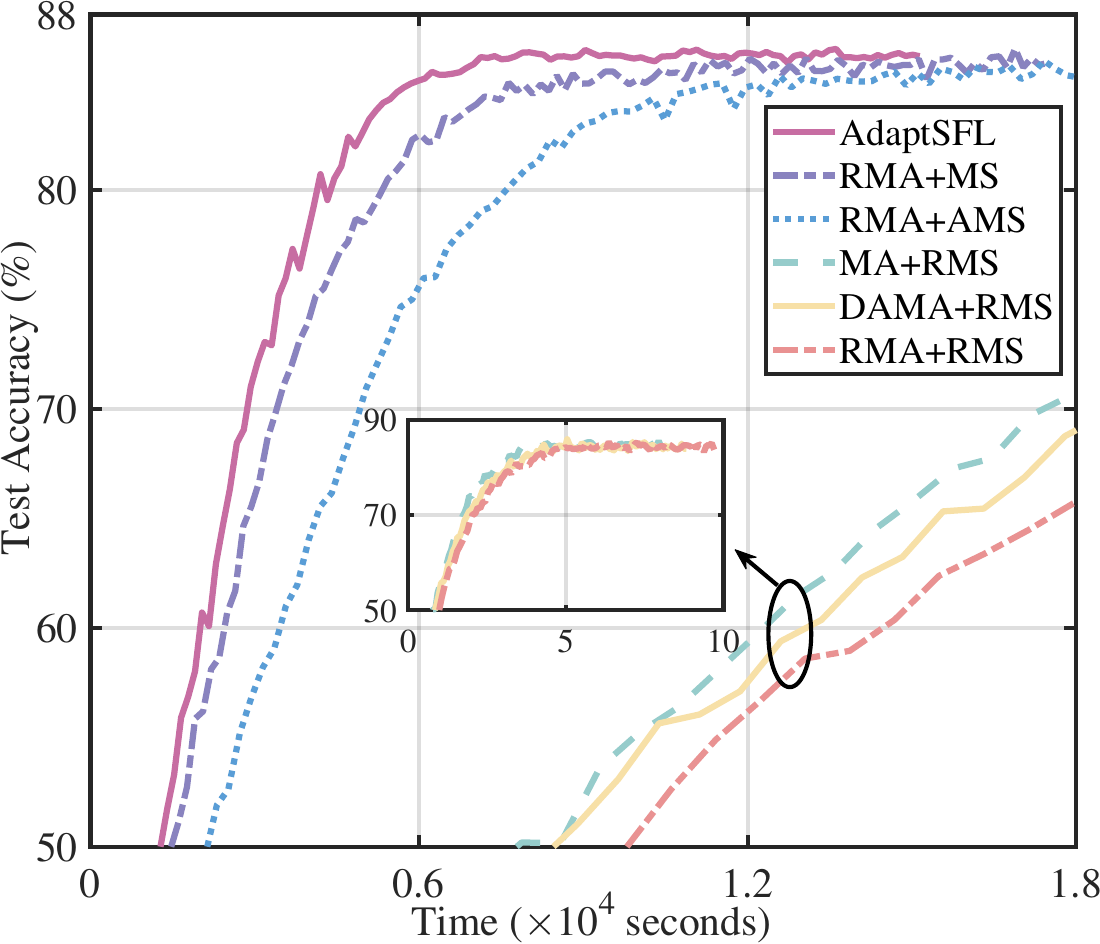}
\label{sfig:cifar_iid_test_accuracy}
}
\subfigure[\rev{CIFAR-10 on VGG-16 under non-IID setting.}]{ \includegraphics[width=.44\columnwidth]{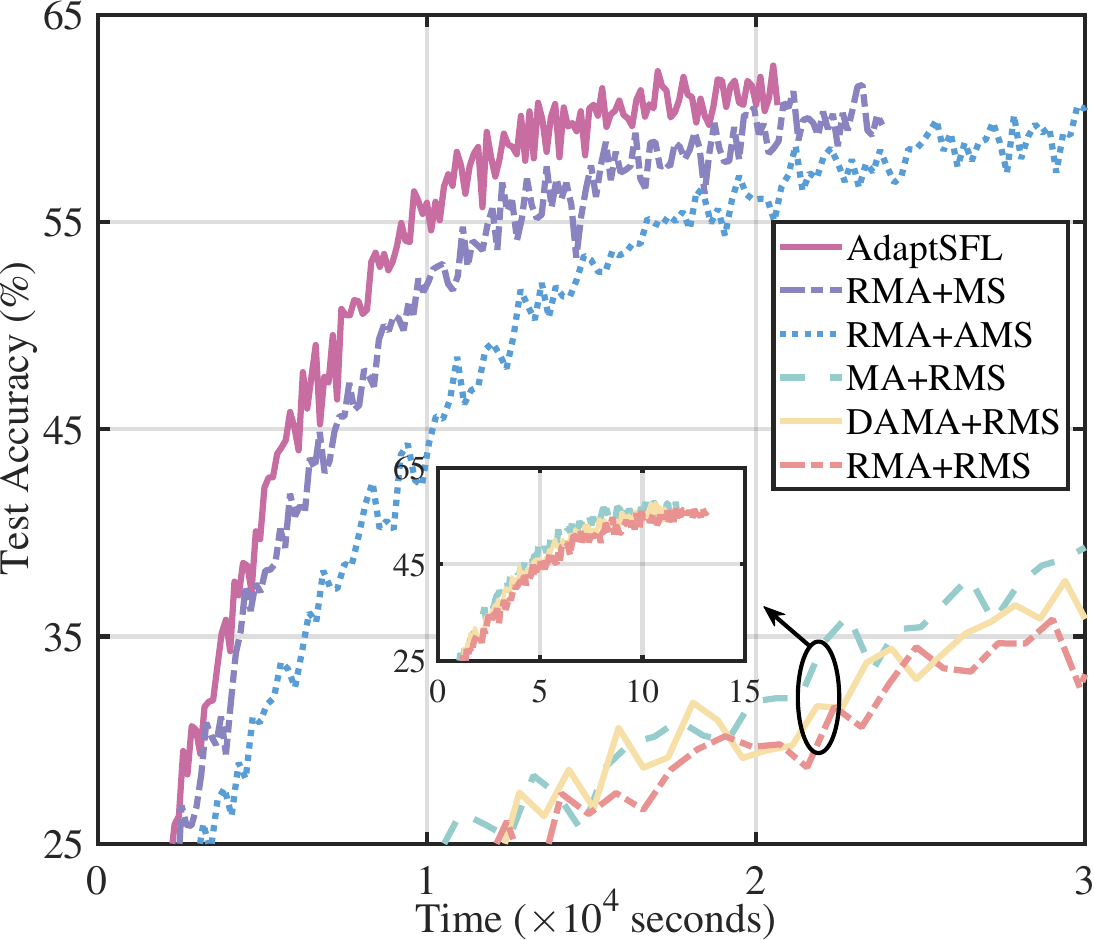}
\label{sfig:cifar_non_iid_test_accuracy}
}
\subfigure[ \rev{CIFAR-100 on ResNet-18 under IID setting.}]{
\includegraphics[width=.44\columnwidth]{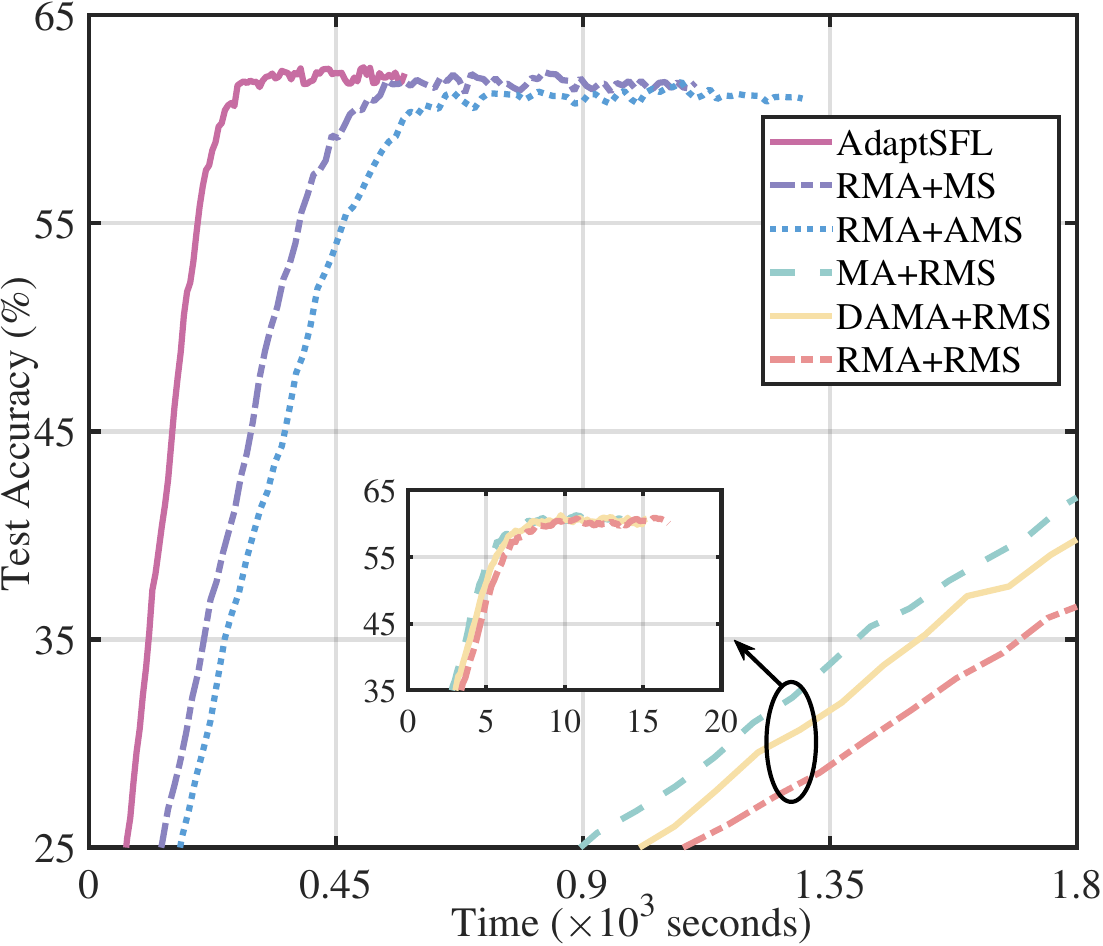}
\label{sfig:mnist_iid_test_accuracy}}
\subfigure[\rev{CIFAR-100 on ResNet-18 under non-IID setting.}]{
    \includegraphics[width=.44\columnwidth]{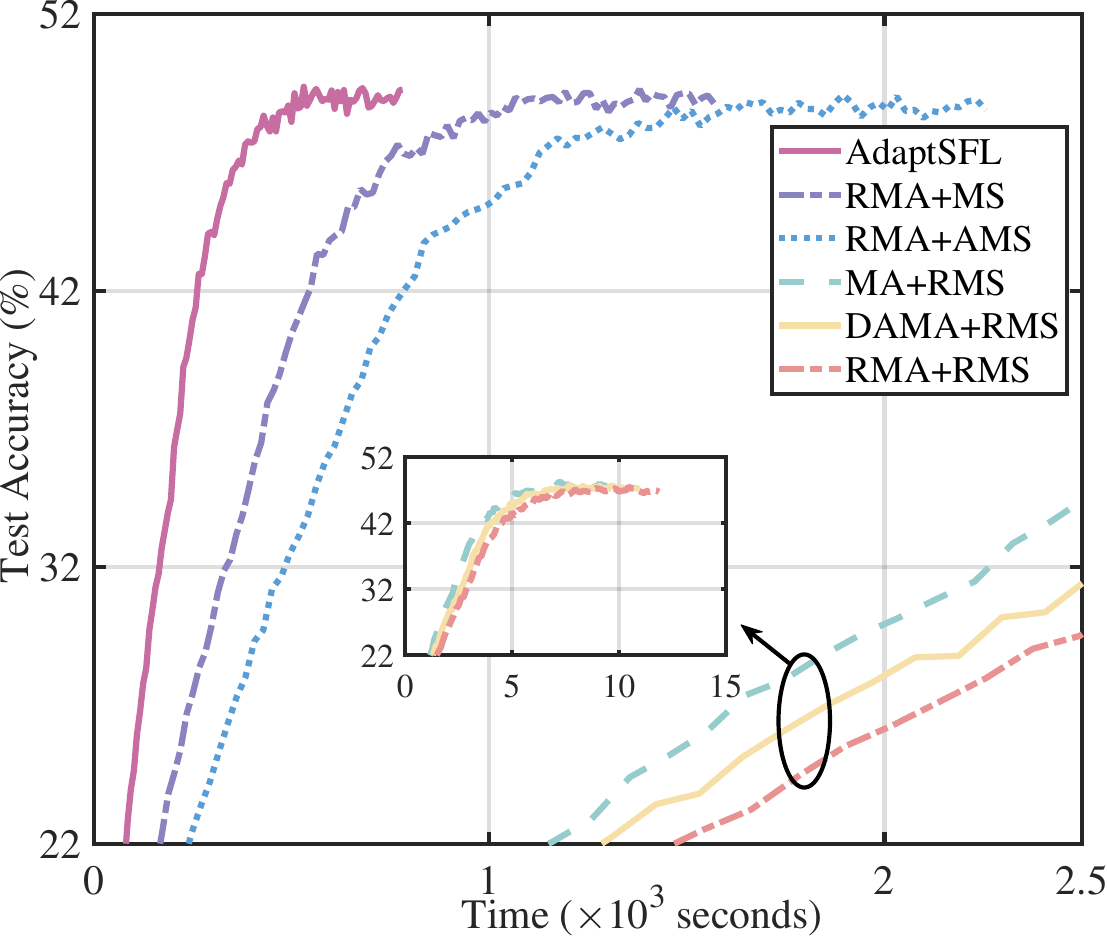}
    \label{sfig:mnist_non_iid_test_accuracy}
}
    \caption{\rev{The training performance for CIFAR-10 and CIFAR-100 datasets under IID and non-IID settings using VGG-16 and ResNet-18.}}
    \label{fig:test_accuracy}
    \vspace{-2ex}
\end{figure}

\begin{table}[t]\label{table_2}
  \centering
  \caption{Simulation Parameters.}
  \renewcommand{\arraystretch}{1.3}{
  \setlength{\tabcolsep}{1.3mm}{
\begin{tabular}{|c|c|c|c|}
\hline
\textbf{Parameter}          & \textbf{Value} & \textbf{Parameter} & \textbf{Value}  \\ \hline
$f_s$             & $20$ TFLOPS             & $f_i$                 & $[1, 2]$ TFLOPS                          \\ \hline
$N$             & 20              & ${r_i^{U}}$/$r_{i,f}^U$           & $[75, 80]$ Mbps                   \\ \hline
${r_i^{D}}$/$r_{i,f}^D$               & $370$ Mbps              & ${r_s^{U}}$/${r_s^{D}}$                  & $400$ Mbps                      \\ \hline
$b$             & 16           & $\gamma$             & $5\times {10^{-4}}$                      \\ \hline

\end{tabular}}}
\end{table}

{We adopt the object detection dataset CIFAR-10 and CIFAR-100~\cite{krizhevsky2009learning} to evaluate the learning performance of AdaptSFL. The CIFAR-10 dataset contains 10 distinct categories of object images, such as airplanes, automobiles and trucks, and comprises 50000 training samples as well as 10000 test samples. The CIFAR-100 dataset includes object images from 100 categories, with each category containing 500 training samples and 100 test samples.} Furthermore, we conduct experiments under IID and non-IID data settings. The data samples are shuffled and distributed evenly to all edge devices in the IID setting. In the non-IID setting~\cite{zhu2019broadband,yang2020energy,lin2024fedsn}, we sort the data samples by labels, divide them into 40 shards, and distribute 2 shards to each of the 20 edge devices.

{We deploy the well-known VGG-16~\cite{simonyan2014very} and ResNet-18~\cite{he2016deep} neural networks. VGG-16 is a classical deep convolutional neural network comprised of 13 convolution layers and 3 fully connected layers, while ResNet-18 is a residual neural network composed of 17 convolutional layers and 1 fully connected layer. } 

% It leverages the stacking of multiple convolution layers to effectively extract features from images, while the fully connected layers are responsible for classifying and predicting the extracted features. 

\subsection{Performance Evaluation of AdaptSFL Framework}\label{simu_setup}

\begin{figure}[t]
% \vspace{-.5ex}
\setlength\abovecaptionskip{3pt}
\centering
\subfigure[CIFAR-10 on VGG-16 under IID setting.]{
\includegraphics[height=3.5cm,width=3.9cm]{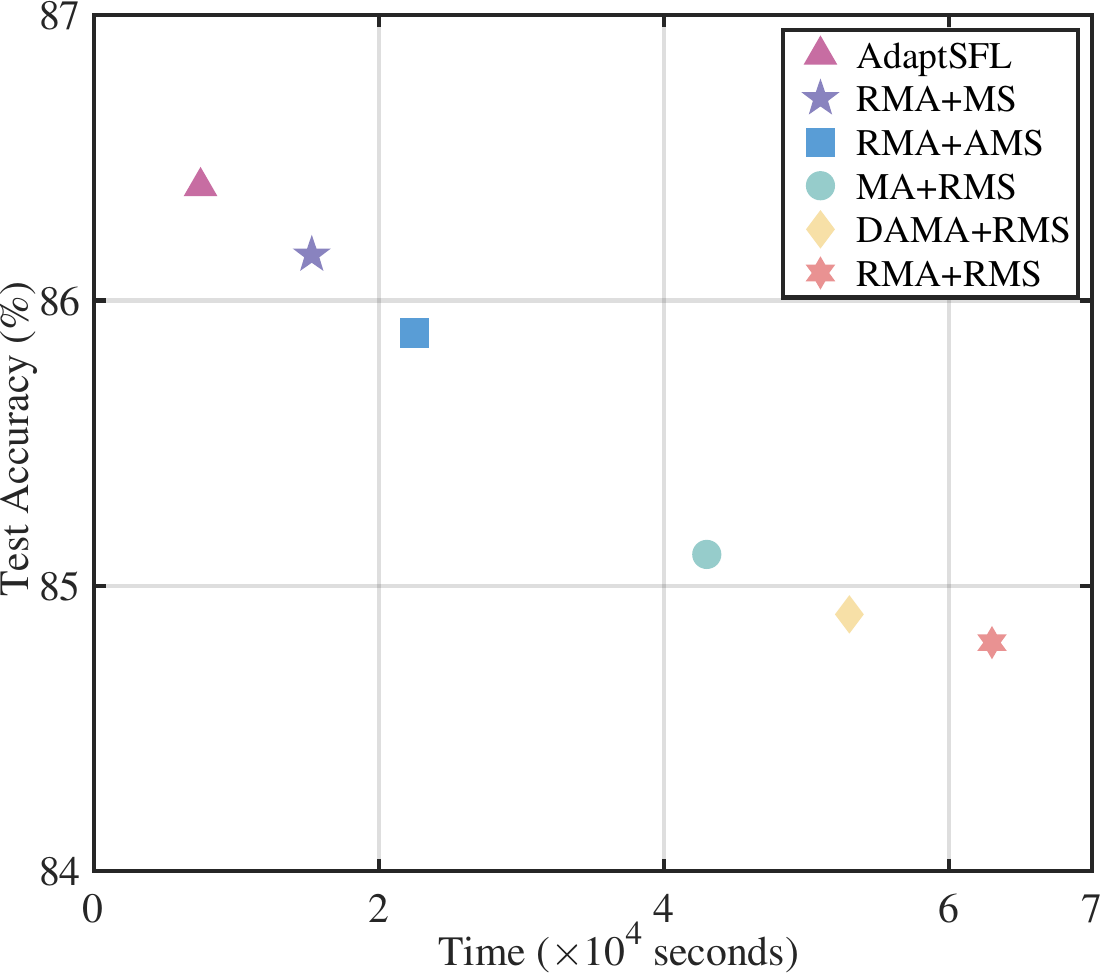}
    \label{sfig:cifar_iid_time_accuracy}
}
\subfigure[CIFAR-10 on VGG-16 under non-IID setting.]{
    \includegraphics[height=3.5cm,width=3.9cm]{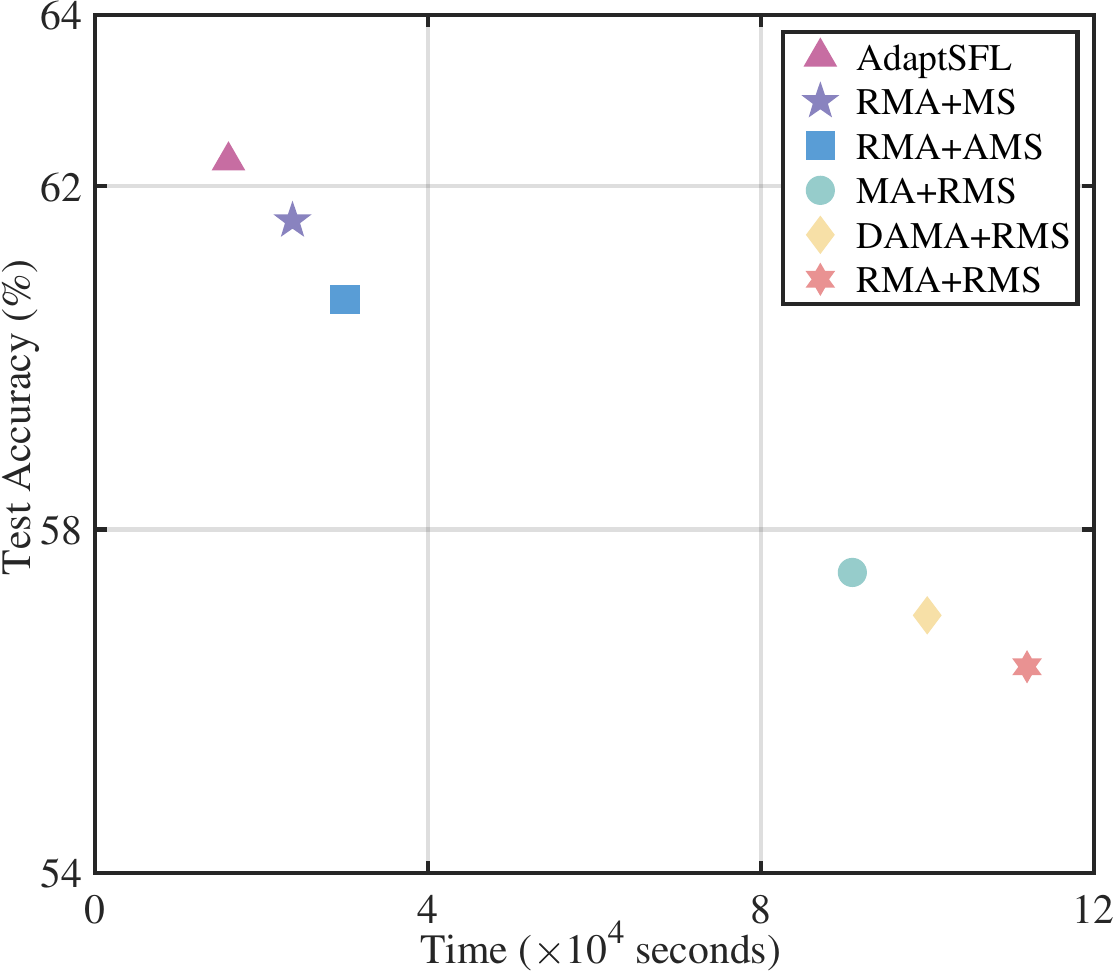}
    \label{sfig:cifar_non_iid_time_accuracy}
}
\subfigure[CIFAR-100 on ResNet-18 under IID setting.]{
\includegraphics[height=3.5cm,width=3.9cm]{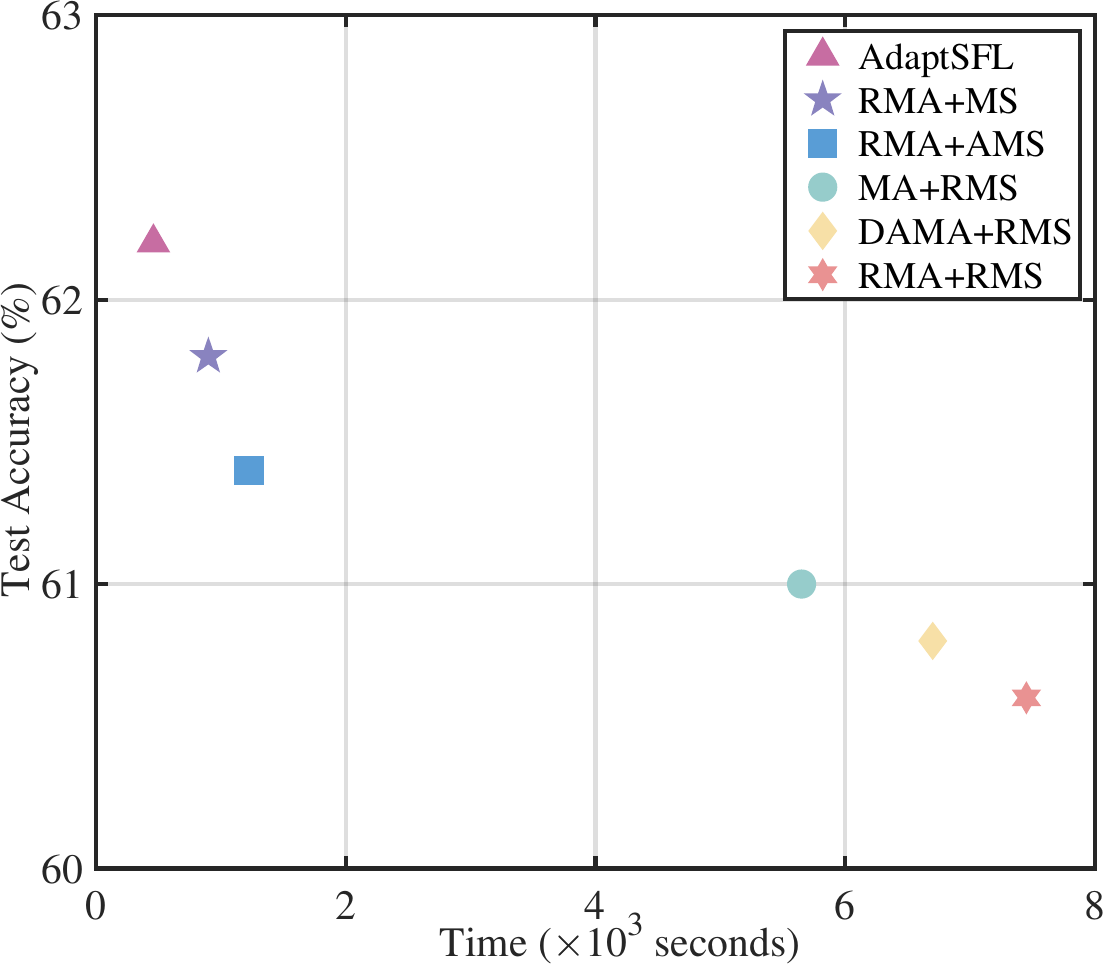}
    \label{sfig:mnist_iid_time_accuracy}
}
\subfigure[CIFAR-100 on ResNet-18 under non-IID setting.]{
    \includegraphics[height=3.5cm,width=3.9cm]{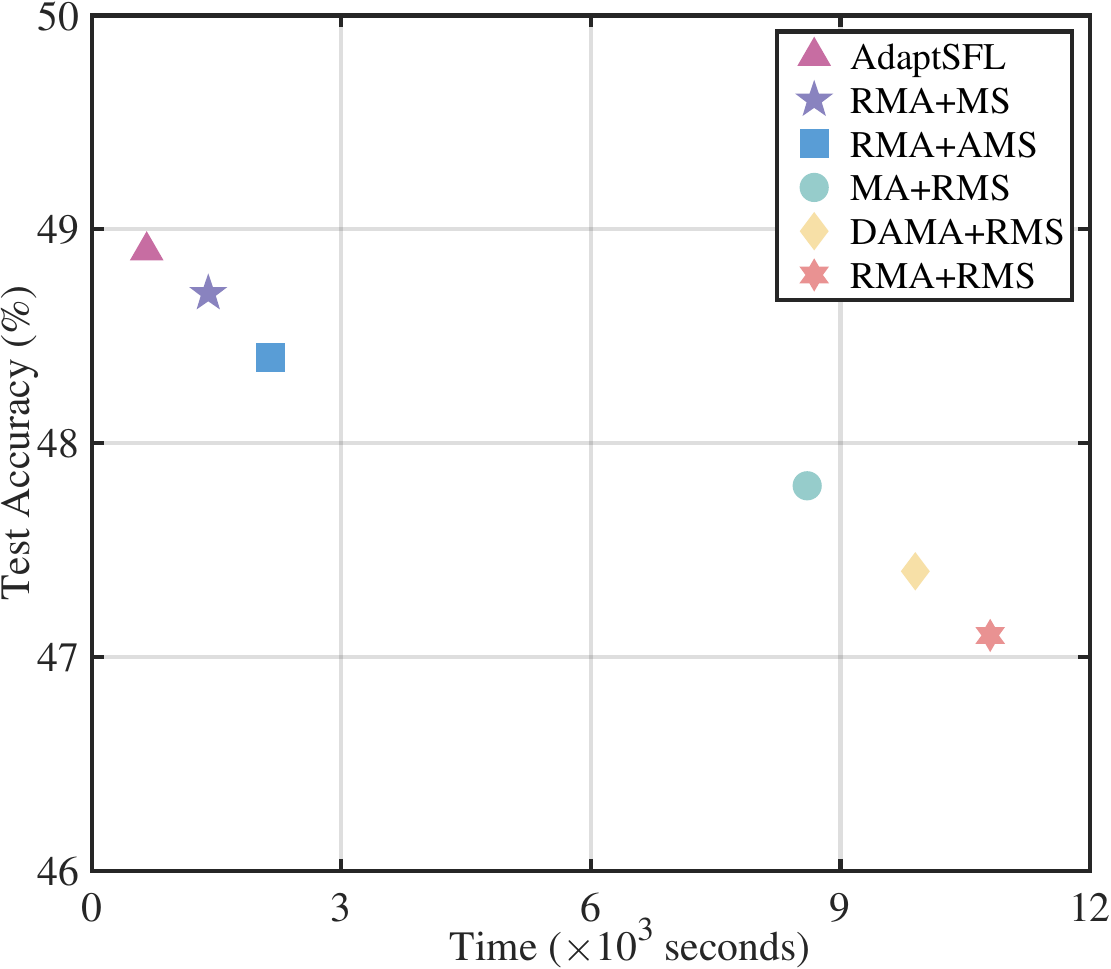}
    \label{sfig:mnist_non_iid_time_accuracy}
}
    \caption{The converged accuracy and time for CIFAR-10 and CIFAR-100 datasets under IID and non-IID settings using VGG-16 and ResNet-18.}
    \label{fig:time_accuracy}
    \vspace{-2ex}
\end{figure}

In this section, we assess the overall performance of AdaptSFL framework in terms of test accuracy and convergence speed. In addition, we evaluate the robustness of AdaptSFL to varying network resources. To investigate the advantages of AdaptSFL framework, we compare it with five benchmarks:
\begin{itemize}
    \item {\bf{RMA+MS:}} The RMA+MS benchmark deploys a random client-side MA strategy (i.e., the client-side MA interval $I$ is randomly drawn from 1 to 25 during model training.), and employs the adaptive MS scheme in Section~\ref{solu_appro}.
    \item {\bf{MA+RMS:}} The MA+RMS benchmark utilizes the adaptive client-side MA strategy in Section~\ref{solu_appro} and adopts a random MS scheme (i.e., randomly selecting model split points during model training).
    \item {\bf{RMA+RMS:}} The RMA+RMS benchmark employs the random client-side MA and MS strategy. 
    \item {{\bf{DAMA+RMS:}} The DAMA+RMS benchmark employs the depth-aware adaptive client-side MA scheme~\cite{you2023aifed} and utilizes the RMS strategy.}
    \item {{\bf{RMA+AMS:}} The RMA+HMS benchmark employs the RMA scheme and utilizes a adaptive resource-heterogeneity-aware MS strategy~\cite{wang2023coopfl}.}

\end{itemize}

{Fig.~\ref{fig:test_accuracy} illustrates the training performance of AdaptSFL and five benchmarks on the CIFAR-10 and CIFAR-100 datasets.} AdaptSFL exhibits faster convergence speed and higher test accuracy compared to the other five benchmarks as the model converges. {Notably, AdaptSFL, RMA+MS, and RMA+AMS converge significantly faster than MA+RMS, DAMA+RMS, and RMA+RMS owing to adaptive MS, which strikes a good balance between communication-computing overhead and training convergence speed.} Furthermore, the comparison between AdaptSFL and RMA+MS and RMA+AMS, shows that the tailored client-side MA strategy can further accelerate model training without deteriorating training accuracy. {The performance gap between RMA+MS and RMA+AMS, as well as MA+RMS and DAMA+RMS, is mainly because RMA+AMS and DAMA+RMS are heuristic and have not been systematically optimized based on the interplay between model convergence and client-side model aggregation and model splitting.} By comparing Fig.~\ref{sfig:cifar_iid_test_accuracy} with Fig.~\ref{sfig:cifar_non_iid_test_accuracy}, and Fig.~\ref{sfig:mnist_iid_test_accuracy} with Fig.~\ref{sfig:mnist_non_iid_test_accuracy}, reveals that the convergence speed of AdaptSFL and other five benchmarks are slower under non-IID setting than under IID setting.

\begin{figure}[t!]
% \vspace{-.5ex}
\setlength\abovecaptionskip{3pt}
\centering
\subfigure[Computing capabilities of edge devices.]{
\includegraphics[height=3.8cm,width=4.1cm]{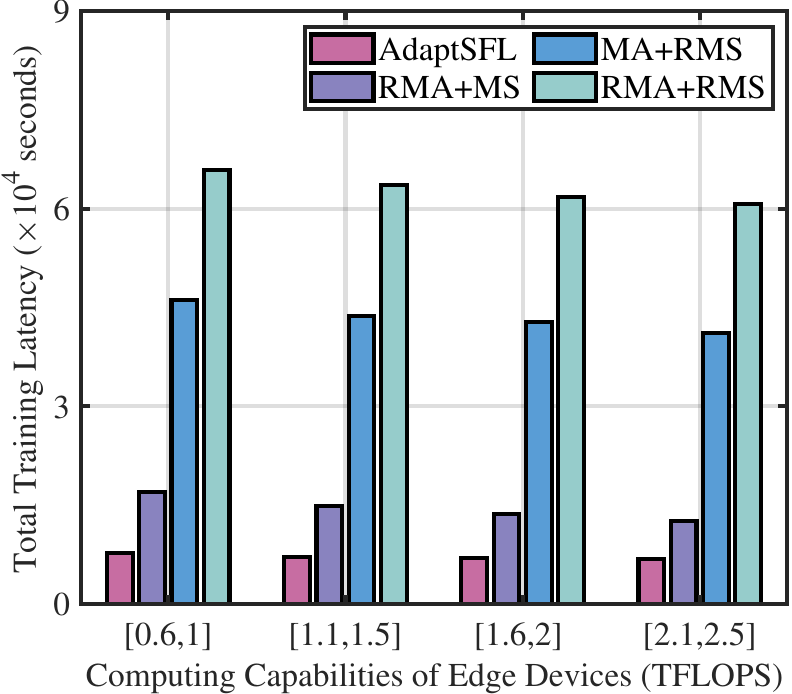}
    \label{sfig:device_comput_accuracy}
}
\subfigure[Computing capabilities of edge server.]{
    \includegraphics[height=3.8cm,width=4.1cm]{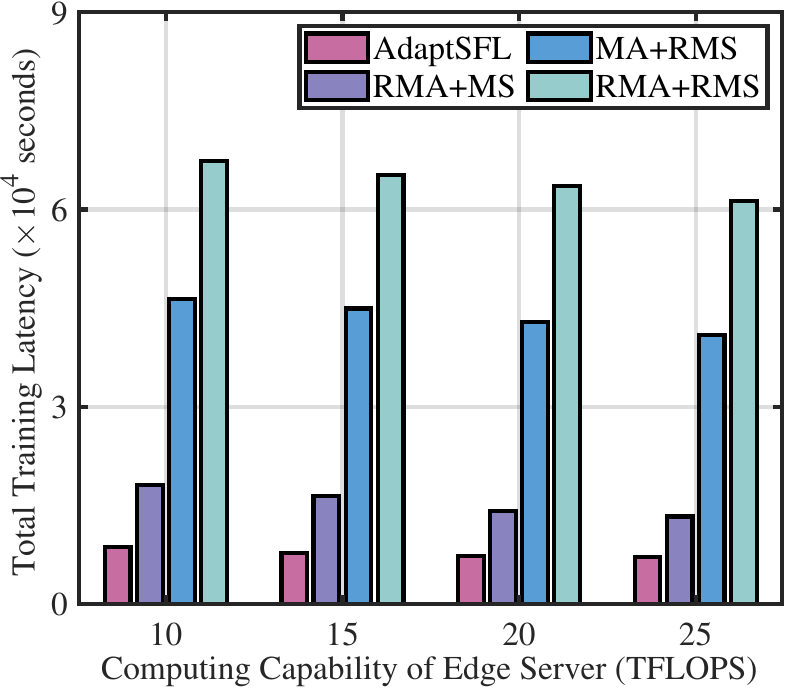}
    \label{sfig:server_comput_accuracy}
}
    \caption{ The converged time versus network computing resources on CIFAR-10 under IID setting using VGG-16.}
    \label{fig:comput_accuracy}
    \vspace{-2ex}
\end{figure}

\begin{figure}[t]
% \vspace{-.5ex}
\setlength\abovecaptionskip{3pt}
\centering
\subfigure[Uplink rates of edge devices.]{
\includegraphics[height=3.7cm,width=4.1cm]{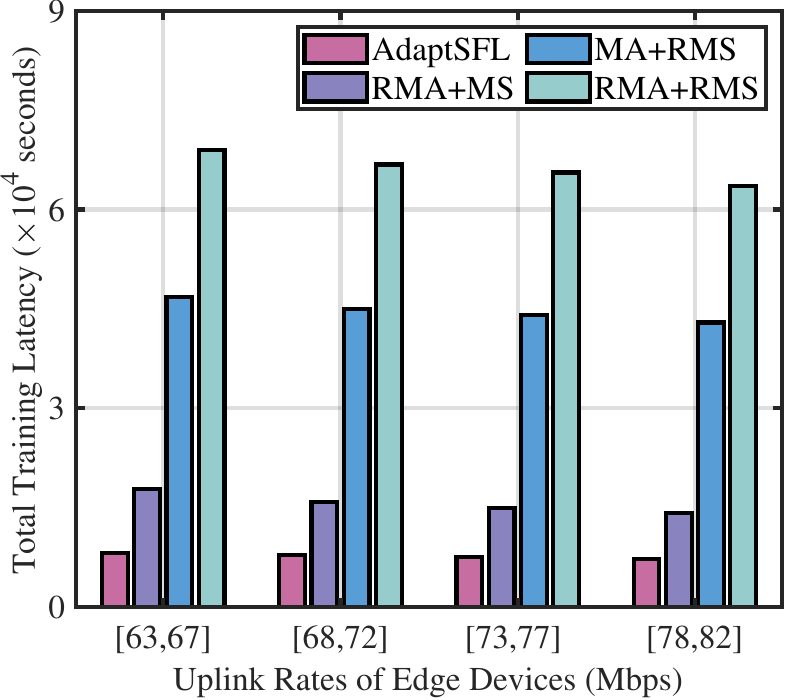}
    \label{sfig:device_commu_accuracy}
}
\subfigure[Inter-server communication rate.]{
    \includegraphics[height=3.7cm,width=4.1cm]{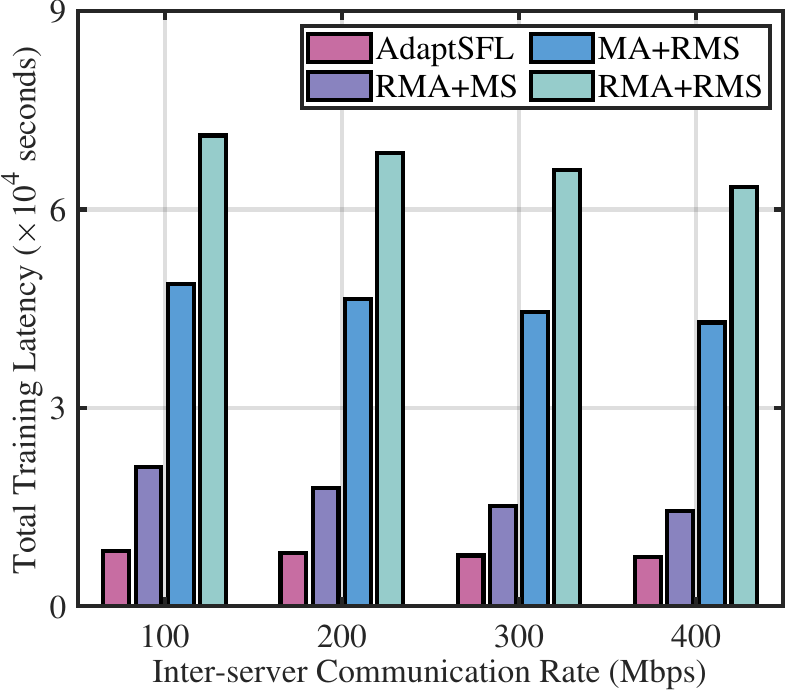}
    \label{sfig:server_commu_accuracy}
}
    \caption{  The converged time versus network communication resources on CIFAR-10 under IID setting using VGG-16.}
    \label{fig:commu_accuracy}
    \vspace{-2ex}
\end{figure}

\begin{figure}[t]
% \vspace{-.5ex}
\setlength\abovecaptionskip{3pt}
\centering
\subfigure[\rev{Computing capability uncertainty.}]{
    \includegraphics[height=3.70cm,width=4.13cm]{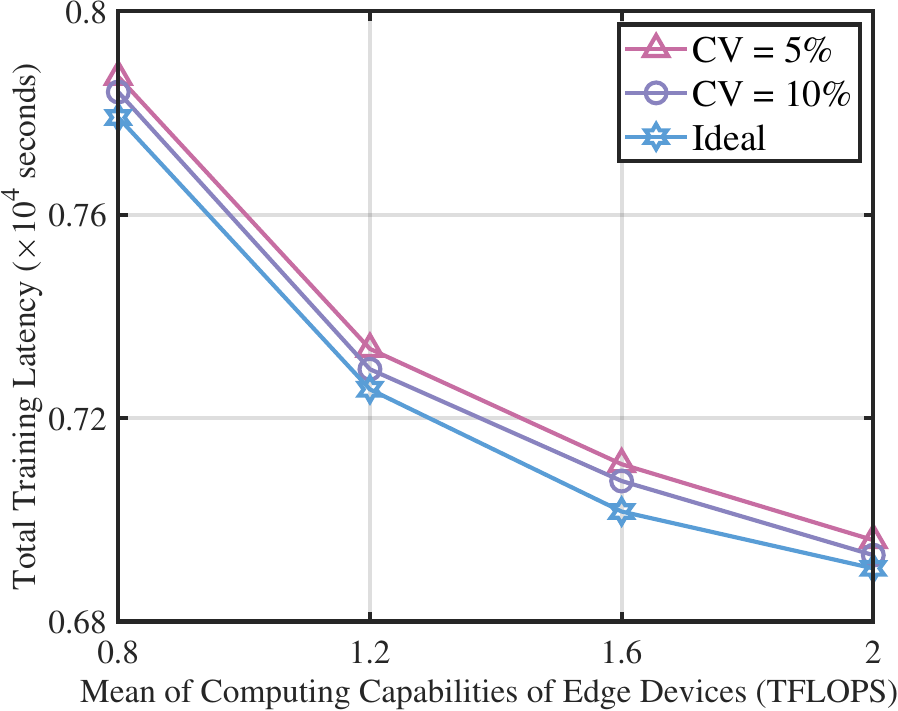}
    \label{sfig:server_commu_accuracy_var}
}
\subfigure[\rev{Uplink transmission rate uncertainty.}]{
\includegraphics[height=3.69cm,width=4.11cm]{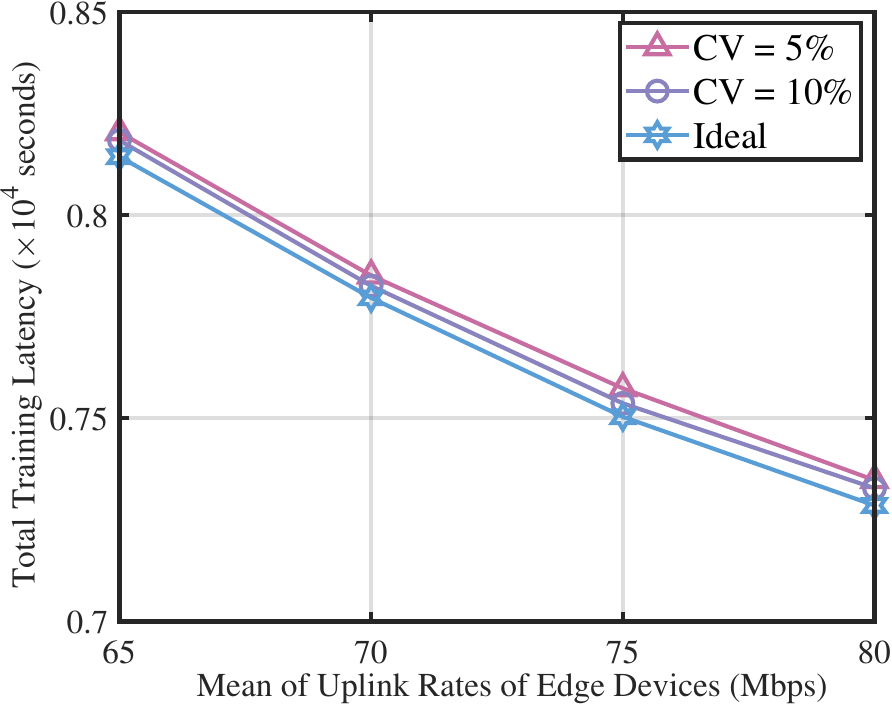}
    \label{sfig:device_commu_accuracy_var}
}
    \caption{ \rev{The impact of network resource measurement uncertainty on converged time on CIFAR-10 under IID setting using VGG-16. }}
    \label{fig:commu_accuracy_var}
    \vspace{-2ex}
\end{figure}

{Fig.~\ref{fig:time_accuracy} presents converged accuracy and time {(i.e., the incremental increase in test accuracy lower than 0.02$\%$ for 5 consecutive training rounds)} for CIFAR-10 and CIFAR-100 datasets.} Comparing AdaptSFL with RMA+MS and MA+RMS reveals a substantial impact of client-side MA interval and MS on converged accuracy and time. Moreover, the impact of MS on model training outweighs that of MA interval. This is because the choice of model split points directly determines the overall aggregation interval of the global model, e.g., shallower model split points (fewer client-side sub-model layers) imply larger portions of the model are aggregated at each training round. In the IID setting, unoptimized MS leads to $1.5 \%$ accuracy degradation and nearly five-fold convergence deceleration on the CIFAR-10 dataset, which is 6.2 and 4.6 times that of unoptimized MA. Furthermore, this difference is more pronounced under the non-IID setting, reaching approximately 7.3 and 10 times, due to model bias caused by discrepancies across local datasets, consistent with FL. The comparison between AdaptSFL and RMA+RMS shows that AdaptSFL achieves a convergence speed improvement of at least 7.7 times over its counterpart without optimization while guaranteeing training accuracy, demonstrating the superior performance of the AdaptSFL framework. {RMA+AMS and DAMA+RMS exhibit slower model convergence and lower accuracy compared to RMA+MS and MA+RMS, primarily because they are not designed based on model convergence.}

Fig.~\ref{fig:comput_accuracy}-\ref{fig:commu_accuracy} show the converged time versus network computing and communication resources on CIFAR-10 under IID setting. AdaptSFL exhibits significantly shorter converged times than the other five benchmarks across varying computing and communication resources. As network resources decline, the convergence speed of RMA+RMS noticeably slows down. This is because it neither optimizes model split points to strike a balance between computing and communication overheads nor optimizes the client-side MA interval for expediting model training. The comparison between RMA+MS and MA+RMS reveals that optimizing MA interval or MS somewhat mitigates the rise in convergence time with reduced network resources.
In contrast, the converged time of AdaptSFL only experiences a slight increase with diminishing network resources, owing to its resource-adaptive client-side MA and MS design. Specifically, as network resources vary, AdaptSFL can optimize MS according to resource conditions and accelerate model convergence by adjusting client-side MA interval to achieve the minimum rounds of model aggregation required for model convergence. This demonstrates the robustness of AdaptSFL to network resources and highlights the adaptability of client-side MA and MS strategies to changes in network resources. \rev{However, in real-world scenarios, computing and communication resources fluctuate during training, leading to discrepancies between measured and actual network resource conditions. To investigate the impact of measurement error, we introduce Gaussian noise with diverse coefficient of variation (CV) to model the fluctuation in data rates and computing capabilities~\cite{esposito2022dts,yoo2024modeling}.  Fig.~\ref{fig:commu_accuracy_var} illustrates that AdaptSFL remains robust across different CVs, causing only minor variations in converged time. This validates the effectiveness of our adaptive MA and MS strategies in dynamic resource fluctuations.}

\begin{figure}[t]
% \vspace{-.5ex}
\setlength\abovecaptionskip{3pt}
\centering
\subfigure[{CIFAR-10 under IID setting.}]{
\includegraphics[height=3.75cm,width=4.1cm]{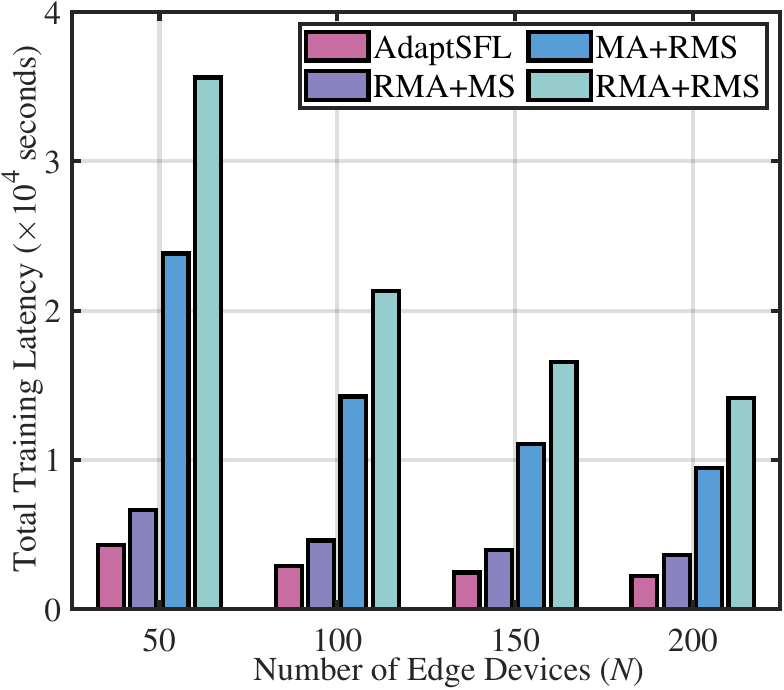}
    \label{sfig:cifar_iid_num}
}
\subfigure[{CIFAR-10 under non-IID setting.}]{
    \includegraphics[height=3.75cm,width=4.1cm]{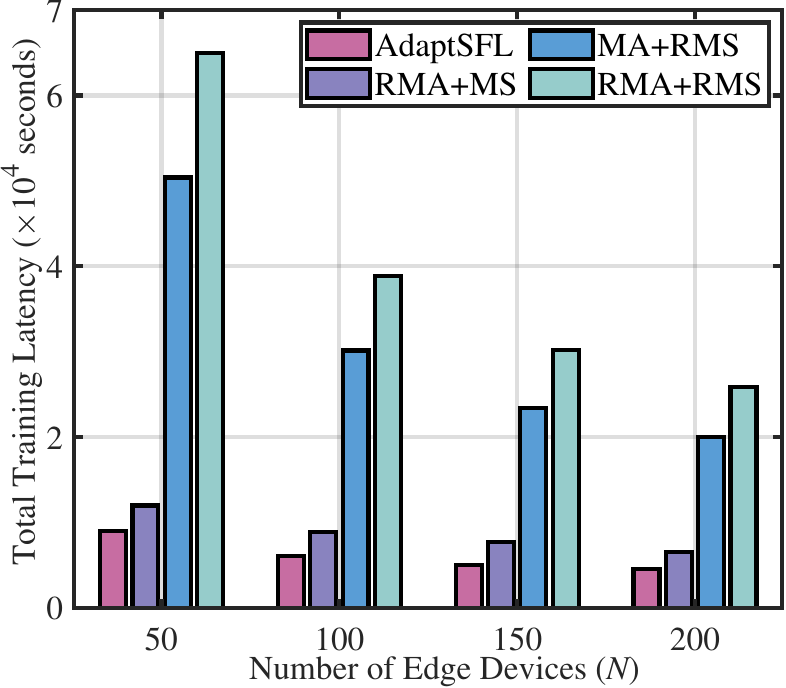}
    \label{sfig:cifar_non_iid_num}
}
    \caption{{The converged time versus number of edge devices on CIFAR-10 under IID and non-IID setting using VGG-16.}}
    \label{fig:cifar_deivce_num}
    \vspace{-2ex}
\end{figure}

{Fig.~\ref{fig:cifar_deivce_num} presents the converged time versus the number of edge devices on CIFAR-10 under IID and non-IID settings. It is seen that AdaptSFL consistently achieves significantly lower training latency compared to MA+RMS, RMA+MS, and RMA+RMS across all configurations of $N$. Notably, as the number of edge devices grows, the performance gap between AdaptSFL and all baselines becomes increasingly pronounced, underscoring the scalability of AdaptSFL in handling large-scale edge networks. Additionally, the total training latency of the non-IID setting is higher than that of the IID setting. This behavior reflects the inherent complexity of training under non-IID data distributions, which often leads to slower convergence due to increased heterogeneity among client-side models. Nonetheless, AdaptSFL maintains its advantage over the baseline methods, demonstrating its robustness and adaptability in tackling diverse and non-uniform data distributions.}

\subsection{Ablation Study of The AdaptSFL Framework}\label{simu_setup}

In this section,  we conduct ablation experiments to illustrate the effectiveness of each component in AdaptSFL.

Fig.~\ref{fig:cifar_cut_aba} shows the impact of MA interval on training performance for the CIFAR-10 dataset. PSL~\cite{lin2023split} exhibits the slowest convergence rate and lowest convergence accuracy due to its lack of client-side MA, leading to inferior generalization performance of local client-side sub-models on the global dataset. In contrast, the proposed resource-adaptive MA strategy expedites model convergence while retaining accuracy comparable to SFL with $I=1$ (equivalent to centralized learning performance). The proposed client-side MA strategy achieves $1.5 \%$ and $5.1 \%$ higher convergence accuracy than PSL, and the convergence time is improved by a factor of 1.8 and 1.9 compared to SFL with $I=1$ under IID and non-IID settings. This improvement stems from its ability to dynamically adjust MA interval to achieve the minimum communication-computing latency required for model convergence. Hence, the effectiveness of the proposed MA scheme is demonstrated.

\begin{figure}[t]
% \vspace{-.5ex}
\setlength\abovecaptionskip{3pt}
\centering
\subfigure[CIFAR-10 under IID setting.]{
\includegraphics[height=3.7cm,width=4.1cm]{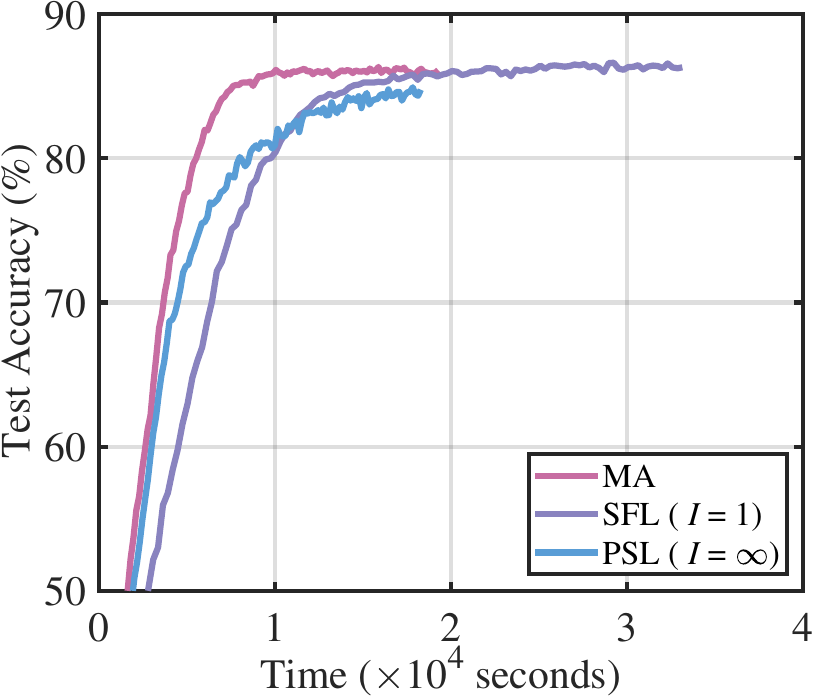}
    \label{sfig:cifar_iid_cut_aba}
}
\subfigure[CIFAR-10 under non-IID setting.]{
    \includegraphics[height=3.7cm,width=4.1cm]{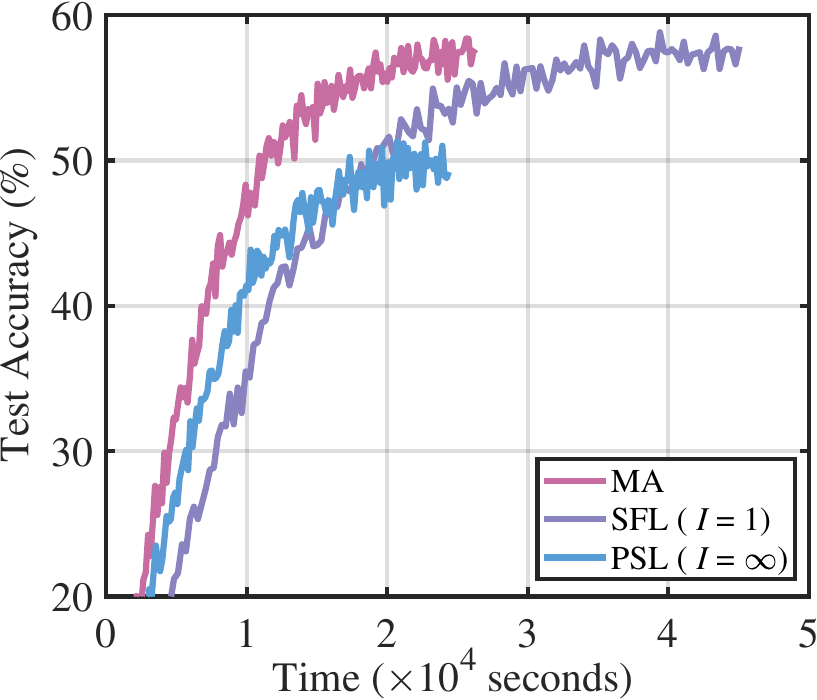}
    \label{sfig:cifar_non_iid_cut_aba}
}
    \caption{Ablation experiments for MA strategy on the CIFAR-10 dataset under IID and non-IID setting using VGG-16 with $L_c=8$, {where parallel split learning (PSL)~\cite{lin2023split} is a special case of SFL when $I = \infty$.}}
    \label{fig:cifar_cut_aba}
    \vspace{-2ex}
\end{figure}

\begin{figure}[t]
% \vspace{-.5ex}
\setlength\abovecaptionskip{3pt}
\centering
\subfigure[CIFAR-10 under IID setting.]{
\includegraphics[height=3.7cm,width=4.1cm]{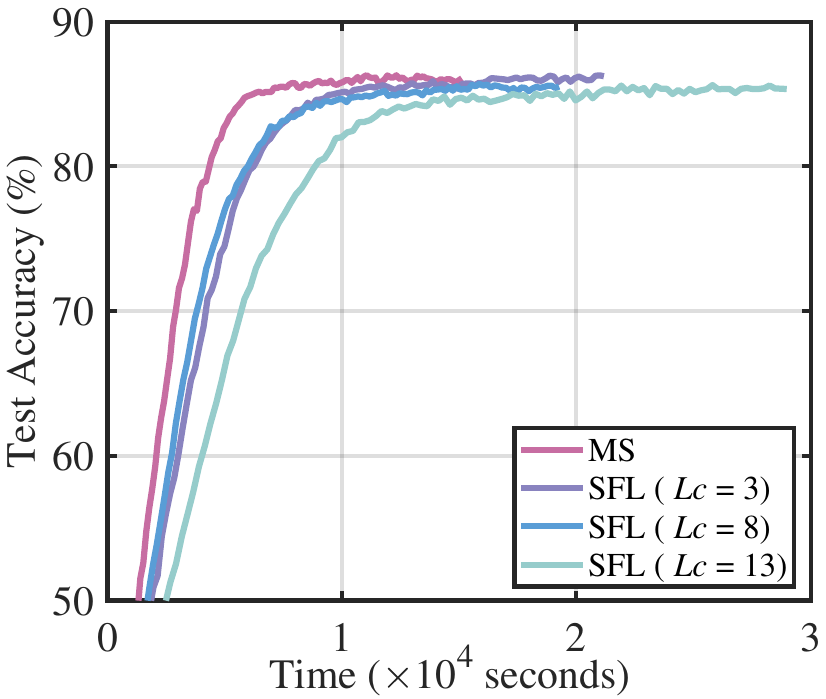}
    \label{sfig:cifar_iid_I_15_different_cut}
}
\subfigure[CIFAR-10 under non-IID setting.]{
    \includegraphics[height=3.7cm,width=4.1cm]{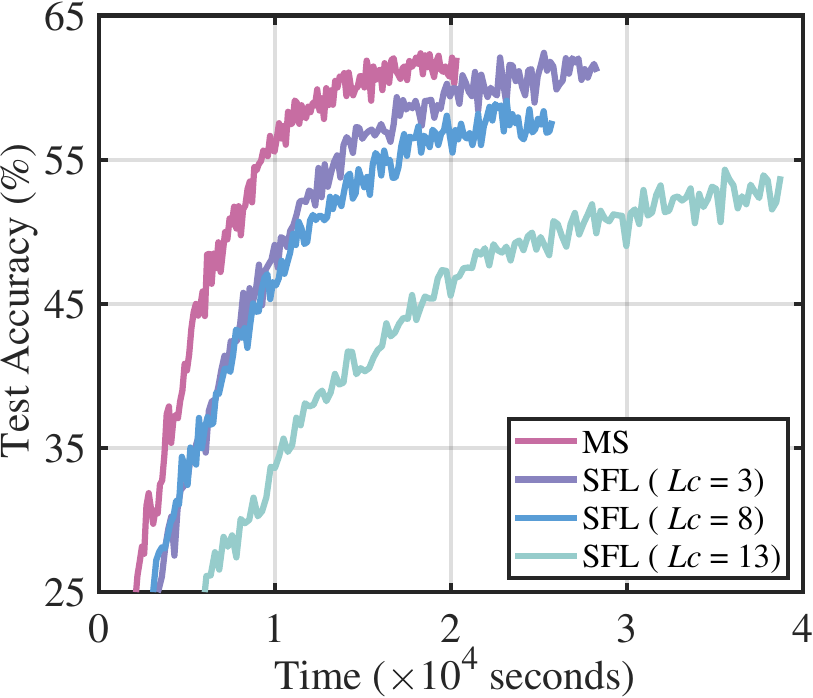}
    \label{sfig:cifar_non_iid_I_15_different_cut}
}
    \caption{Ablation experiments for MS scheme on the CIFAR-10 dataset under IID and non-IID setting using VGG-16 with $I=15$.}
    \label{fig:I_15_different_cut}
    \vspace{-2ex}
\end{figure}

Fig.~\ref{fig:I_15_different_cut} presents the impact of MS on training performance for the CIFAR-10 dataset. It can be seen that the convergence accuracy and speed decrease with the deepening of model split point $L_c$, consistent with the derived convergence boundary in Eqn.~\eqref{convergence_bound}. This is because deeper model split points result in the smaller portion of global model being aggregated per training round, leading to a lower overall update frequency of the whole model. Furthermore, the impact of MS on convergence accuracy and speed is more significant under non-IID than IID settings. This is attributed to the higher sensitivity of local dataset distribution discrepancies to MS under non-IID setting. The comparisons between MS and the other five benchmarks reveal that the proposed MS strategy can accelerate model convergence while guaranteeing training performance. This demonstrates the superior performance of the tailored MS strategy.

\section{Conclusions}\label{conclu}

In this paper, we proposed a split federated learning (SFL) framework, named AdaptSFL, to minimize the training latency of SFL over edge networks. We first derived a convergence bound of SFL in terms of model aggregation intervals and model split points. Then, we devised a resource-adaptive client-side model aggregation (MA) scheme and a model splitting (MS) scheme to minimize end-to-end training latency subject to a training loss constraint under resource and data heterogeneity. Simulation results demonstrated that our proposed AdaptSFL framework achieves the target accuracy with significantly less time budget compared to benchmarks.

While this work has demonstrated the potential of integrating SFL and edge computing paradigm, more research efforts could be made to further enhance the efficiency of SFL. The optimization of other key training hyperparameters, such as batch size and learning rate, is worth further exploration. Moreover, future work could also incorporate network and computing probing costs into the optimization framework. This is particularly pertinent in highly dynamic network environments, where frequent state changes might exacerbate the relative impact of probing overhead. {\rev{In addition, further research is needed to prototype SFL in real-world edge computing systems to further validate training performance under practical deployment constraints.}} All these considerations could further enhance the efficiency of SFL under resource-constrained systems.

% In this paper, we have proposed a novel SL framework,
% efficient parallel split learning (EPSL), to accelerate the model
% training. EPSL parallelizes the client-side model training, and
% lowers the dimension of activations’ gradients and server’s
% computation workload by performing the last-layer gradient
% aggregation, leading to significant reduction in training latency.
% By considering the heterogeneity in channel conditions and
% computing capabilities at client devices, we have designed
% a resource management and layer split strategy to jointly
% optimize subchannel allocation, power control, and cut layer
% selection to minimize the training latency for ESPL over the
% wireless edge networks. While the formulated problem is a
% mixed-integer nonlinear programming, we decompose it into
% four less complicated subproblems based on different decision
% variables, and present an efficient BCD-based algorithm to
% solve this problem. Simulation results demonstrate that our
% proposed EPSL framework takes significantly less time budget
% to achieve a target accuracy than existing benchmarks, and the
% effectiveness of the tailored resource management and layer
% split strategy.
% This work has demonstrated the potential of integrating
% EPSL and edge computing paradigm. However, more research
% attention needs to be paid to addressing label privacy issue
% (especially for privacy-sensitive data, such as disease diagnosis
% data) resulting from the label sharing in SL. Furthermore,
% convergence analysis of EPSL is worth further exploration.

\appendix

% \begin{appendices}
%       \section{ \textit{A. Proof of Lemma 1} }\label{aa}
%       some text in Appendix A
%       \section{  }
%       some text in Appendix B
% \end{appendices}

\section*{{A. Proof of Lemma 1}} \label{aa}
Fix training round $t\geq 1$. Considering the largest $t_{0} \leq  t$ that satisfies $t_0 \bmod I = 0$ (Note that such $t_{0}$ must exist and $t - t_{0} \leq I$.) Recalling the Eqn.~\eqref{non_commen_update}, Eqn.~\eqref{client_side_update} and Eqn.~\eqref{g_ci} for updating the model weights, that
\begin{align}\label{lemma_eq_1}
\mathbf{h}_{c,i}^{t}=  \mathbf{h}_{c}^{t_0}  -\gamma \!\!\sum_{\tau = t_{0}+1}^{t}\mathbf{g}_{c,i}^{\tau}
\end{align}
{By \eqref{h_c_define},} we have
\begin{align*}
{\mathbf{h}}_c^{t}  = {\mathbf{h}}_c^{t_0} -\gamma\!\!\sum_{\tau = t_{0}+1}^{t} \frac{1}{N}\sum_{i=1}^{N}\mathbf{g}_{c,i}^{\tau}
\end{align*}

Thus, we have
{
\begin{align*}
&\mathbb{E} [\Vert \mathbf{h}^{t}_{c} - {\mathbf{h}}^{t}_{c,i} \Vert^{2}] \\=&  {\mathbbm{1}}_{\{I > 1\}} \mathbb{E} [\Vert \gamma\!\! \sum_{\tau = t_{0}+1}^{t} \frac{1}{N}\sum_{i=1}^{N}\mathbf{g}_{c,i}^{\tau} -\gamma\!\! \sum_{\tau = t_{0}+1}^{t}\!\!\!\!\mathbf{g}_{c,i}^{\tau} \Vert^{2}]\\
=& {\mathbbm{1}}_{\{I > 1\}} \gamma^{2} \mathbb{E}  [\Vert \sum_{\tau = t_{0}+1}^{t} \frac{1}{N}\sum_{i=1}^{N}\mathbf{g}_{c,i}^{\tau} -\!\!\sum_{\tau = t_{0}+1}^{t}\!\!\!\!\mathbf{g}_{c,i}^{\tau}\Vert^{2}]\\
\overset{(a)}{\leq}& {\mathbbm{1}}_{\{I > 1\}} 2\gamma^{2} \mathbb{E} [\Vert \sum_{\tau = t_{0}+1}^{t} \!\frac{1}{N}\sum_{i=1}^{N}\mathbf{g}_{c,i}^{\tau}\Vert^{2} + \Vert \!\!\!\sum_{\tau = t_{0}+1}^{t}\!\!\!\!\mathbf{g}_{c,i}^{\tau}\Vert^{2} ]\\
\overset{(b)}{\leq}& {\mathbbm{1}}_{\{I > 1\}} 2\gamma^{2} (t-t_{0}) \mathbb{E} [ \sum_{\tau = t_{0}+1}^{t}\!\Vert  \frac{1}{N}\sum_{i=1}^{N}\mathbf{g}_{c,i}^{\tau}\Vert^{2} + \!\!\!\!\sum_{\tau = t_{0}+1}^{t}\!\!\Vert \mathbf{g}_{c,i}^{\tau}\Vert^{2}]\\
\overset{(c)}{\leq}& {\mathbbm{1}}_{\{I > 1\}} 2\gamma^{2} (t-t_{0}) \mathbb{E} [\sum_{\tau = t_{0}+1}^{t} (\frac{1}{N}\sum_{i=1}^{N}\Vert \mathbf{g}_{c,i}^{\tau}\Vert^{2}) +\!\!\sum_{\tau = t_{0}+1}^{t}\!\!\!\Vert \mathbf{g}_{c,i}^{\tau}\Vert^{2}]\\
 \overset{(d)}{\leq} & {\mathbbm{1}}_{\{I > 1\}} 4\gamma^{2} I^{2} \sum\limits_{j = 1}^{L_c}  {G _j^2}
\end{align*}
}%
where (a)-(c) follows by using the inequality $\Vert \sum_{i=1}^{n} \mathbf{z}_{i}\Vert^{2} \leq n \sum_{i=1}^{n} \Vert \mathbf{z}_{i}\Vert^{2}$ for any vectors $\mathbf{z}_{i}$ and any positive integer $n$ (using $n=2$ in (a), $n=t-t_0$ in (b), and $n=N$ in (c)); and (d) follows from {\bf Assumption \ref{asp:2}}.

\section*{{B. Proof of Theorem 1}}\label{bb}
For training round $t\geq 1$. By the smoothness of loss function $f\left(  \cdot  \right)$, we have
{
\begin{align}\label{eq:equal_1_total}
\mathbb{E}[ {f({{\bf{w}}^t})} ] \le& {\rm{ }}\mathbb{E}[ {f({{\bf{w}}^{t - 1}})} ] + \mathbb{E}[ {\langle {\nabla _{\bf{w}}}f({{\bf{w}}^{t - 1}}),{{\bf{w}}^t} - {{\bf{w}}^{t - 1}}\rangle } ]{\rm{  }} \nonumber \\&~+ \frac{\beta }{2}\mathbb{E}[ {{{\| {{{\bf{w}}^t} - {{\bf{w}}^{t - 1}}} \|}^2}} ].
\end{align}
}

Note that
{\begin{align}\label{eq:w_decouple}
 &\mathbb{E}[ {{{\| {{{\bf{w}}^t} - {{\bf{w}}^{t - 1}}} \|}^2}} ]\nonumber\\
 =&\mathbb{E}[ {{{\| {[ {{\bf{h}}_c^t;{\bf{h}}_s^t} ] - [ {{\bf{h}}_c^{t - 1};{\bf{h}}_s^{t - 1}} ]} \|}^2}} ]\nonumber\\ =&\mathbb{E}[ {{{\| {[ {{\bf{h}}_c^t - {\bf{h}}_c^{t - 1};{\bf{h}}_s^t - {\bf{h}}_s^{t - 1}} ]} \|}^2}} ]\nonumber\\=&\mathbb{E}[ {{{\| { {{\bf{h}}_c^t - {\bf{h}}_c^{t - 1}} } \|}^2}} ] + \mathbb{E}[ {{{\| { {{\bf{h}}_s^t - {\bf{h}}_s^{t - 1}} } \|}^2}} ],
\end{align}}%
where $\mathbb{E}[ {{{\| { {{\bf{h}}_c^t - {\bf{h}}_c^{t - 1}} } \|}^2}} ]$ can be bounded as 

{
\begin{align}\label{eq:wc_squre}
 &\mathbb{E}[ {{{\| { {{\bf{h}}_c^t - {\bf{h}}_c^{t - 1}} } \|}^2}} ] \overset{(a)}{=}{\gamma ^2}\mathbb{E}[||\frac{1}{N}\sum\limits_{i = 1}^N {{\bf{g}}_{c,i}^t} |{|^2}] \nonumber\\ 
 \overset{(b)}{=}& {\gamma ^2}\mathbb{E}[||\frac{1}{N}\sum\limits_{i = 1}^N {\left( {{\bf{g}}_{c,i}^t - {\nabla _{{{\bf{h}}_c}}}{f_i}\left( {{\bf{h}}_{c,i}^{t - 1}} \right)} \right)} |{|^2}]\nonumber\\ 
 +& {\gamma ^2}\mathbb{E}[||\frac{1}{N}\sum\limits_{i = 1}^N {{\nabla _{{{\bf{h}}_c}}}{f_i}\left( {{\bf{h}}_{c,i}^{t - 1}} \right)} |{|^2}] \nonumber\\
\overset{(c)}{=}& \frac{{{\gamma ^2}}}{{{N^2}}}\sum\limits_{i = 1}^N  \mathbb{E}[||{\bf{g}}_{c,i}^t - {\nabla _{{{\bf{h}}_c}}}{f_i}\left( {{\bf{h}}_{c,i}^{t - 1}} \right)|{|^2}]\nonumber\\
+& {\gamma ^2}\mathbb{E}[||\frac{1}{N}\sum\limits_{i = 1}^N {{\nabla _{{{\bf{h}}_c}}}{f_i}\left( {{\bf{h}}_{c,i}^{t - 1}} \right)} |{|^2}] \nonumber\\
\overset{(d)}{\leq }& \frac{{{\gamma ^2} \sum\limits_{j = 1}^{L_c} {\sigma _j^2}}}{N} +{\gamma ^2}\mathbb{E}[||\frac{1}{N}\sum\limits_{i = 1}^N {{\nabla _{{{\bf{h}}_c}}}{f_i}\left( {{\bf{h}}_{c,i}^{t - 1}} \right)} |{|^2}].
\end{align}}%
where (a) follows from {Eqn.~\eqref{h_c_define}} and Eqn.~\eqref{lemma_eq_1}; (b) follows by observing that $\mathbb{E}[\mathbf{g}_{c,i}^{t}] = \nabla_{{\bf{h}}_c} f_{i}({\mathbf{h}}_{c,i}^{t-1})$ and applying the equality $\mathbb{E}[\Vert \mathbf{z} \Vert^{2}] = \mathbb{E} [ \Vert \mathbf{\mathbf{z}} - \mathbb{E}[\mathbf{z}]\Vert^{2}] + \Vert\mathbb{E}[\mathbf{z}] \Vert^{2}$ that holds for any random vector $\mathbf{z}$; (c) follows because each $\mathbf{g}_{c,i}^{t} - \nabla_{{\bf{h}}_c} f_{i}(\mathbf{h}_{c,i}^{t-1})$ has zero mean and is independent across edge devices; and (d) follows from {\bf Assumption \ref{asp:2}}.

Similarly, $\mathbb{E}[ {{{\| { {{\bf{h}}_s^t - {\bf{h}}_s^{t - 1}} ]} \|}^2}} $ has an upper bound:
{
\small \begin{align}\label{eq:ws_squre}
\mathbb{E}[ {{{\| { {{\bf{h}}_s^t - {\bf{h}}_s^{t - 1}} } \|}^2}}] \le \frac{{{\gamma ^2}\sum\limits_{j = L_c+1}^{L} {\sigma _j^2}}}{N} + {\gamma ^2}\mathbb{E}[||\frac{1}{N}\sum\limits_{i = 1}^N {{\nabla _{{{\bf{h}}_s}}}{f_i}\left( {{\bf{h}}_{s,i}^{t - 1}} \right)} |{|^2}].
\end{align}}

Substituting Eqn.~\eqref{eq:wc_squre} and Eqn.~\eqref{eq:ws_squre} into Eqn.~\eqref{eq:w_decouple} yields
{\begin{align}\label{eq:w-difference-squre}
&\mathbb{E}[\|{{\bf{w}}^t} - {{\bf{w}}^{t - 1}}\|{^2}] \nonumber \\\le& \frac{{{\gamma ^2}\sum\limits_{j = 1}^{L} {\sigma _j^2}}}{N} + {\gamma ^2}\mathbb{E}[||\frac{1}{N}\sum\limits_{i = 1}^N {{\nabla _{{{\bf{h}}_c}}}{f_i}\left( {{\bf{h}}_{c,i}^{t - 1}} \right)} |{|^2}] \nonumber \\
+& {\gamma ^2}\mathbb{E}[||\frac{1}{N}\sum\limits_{i = 1}^N {{\nabla _{{{\bf{h}}_s}}}{f_i}\left( {{\bf{h}}_{s,i}^{t - 1}} \right)} |{|^2}]\nonumber\\
=&\frac{{{\gamma ^2}\sum\limits_{j = 1}^{L} {\sigma _j^2}}}{N} + {\gamma ^2}\mathbb{E}[||\frac{1}{N}\sum\limits_{i = 1}^N {{\nabla _{\bf{w}}}{f_i}\left( {{\bf{w}}_i^{t - 1}} \right)} |{|^2}]
\end{align}}

We further note that 
{
\begin{align} \label{eq:inner_product_w}
&\mathbb{E}[\langle \nabla_{\bf{w}} f({\mathbf{w}}^{t-1}), {\mathbf{w}}^{t} - {\mathbf{w}}^{t-1}\rangle] \nonumber\\
\overset{(a)}{=}& -\gamma \mathbb{E} [\langle \nabla_{\bf{w}} f({\mathbf{w}}^{t-1}), \frac{1}{N} \sum_{i=1}^{N} \mathbf{g}_{i}^{t}\rangle] \nonumber \\
\overset{(b)}{=}& -\gamma \mathbb{E}[\langle {\nabla_{\bf{w}}}f({{\bf{w}}^{t - 1}}),\frac{1}{N}\sum\limits_{i = 1}^N \nabla_{\bf{w}}  {f_i}({\bf{w}}_i^{t - 1})\rangle ]\nonumber \\
\overset{(c)}=&  - \frac{\gamma }{2}\mathbb{E}[||{\nabla _{\bf{w}}}f({{\bf{w}}^{t - 1}})|{|^2} + ||\frac{1}{N}\sum\limits_{i = 1}^N {{\nabla _{\bf{w}}}} {f_i}({\bf{w}}_{i}^{t - 1})|{|^2} \nonumber \\
&- ||{\nabla _{{{\bf{w}}}}}f({{\bf{w}}^{t - 1}}) - \frac{1}{N}\sum\limits_{i = 1}^N {{\nabla _{{{\bf{w}}}}}} {f_i}({\bf{w}}_{i}^{t - 1})|{|^2}]
\end{align}}%
where (a) follows from ${{\bf{w}}^t} = \frac{{{1}}}{N} \sum\limits_{i = 1}^N {\bf{w}}_i^t$; (c) follows from the  identity $\langle \mathbf{z}_{1}, \mathbf{z}_{2}\rangle = \frac{1}{2} \big( \Vert \mathbf{z}_{1}\Vert^{2} + \Vert \mathbf{z}_{2}\Vert^{2} - \Vert \mathbf{z}_{1} - \mathbf{z}_{2}\Vert^{2} \big)$ for any two vectors $\mathbf{z}_{1}, \mathbf{z}_{2}$ of the same length; (b) follows from 
{
	\begin{align*}
	&\mathbb{E}[\langle \nabla_{\bf{w}} f({\mathbf{w}}^{t-1}), \frac{1}{N} \sum_{i=1}^{N} \mathbf{g}_{i}^{t}\rangle] \\
	=& \mathbb{E}[\mathbb{E}[\langle \nabla_{\bf{w}} f({\mathbf{w}}^{t-1}), \frac{1}{N} \sum_{i=1}^{N} \mathbf{g}_{i}^{t}\rangle | \boldsymbol{\xi}^{[t-1]}]] \\
	=& \mathbb{E}[\langle \nabla_{\bf{w}} f({\mathbf{w}}^{t-1}), \frac{1}{N} \sum_{i=1}^{N} \mathbb{E}[\mathbf{g}_{i}^{t}| \boldsymbol{\xi}^{[t-1]}]\rangle ]\\
	 =& \mathbb{E}[\langle \nabla_{\bf{w}} f({\mathbf{w}}^{t-1}), \frac{1}{N} \sum_{i=1}^{N} \nabla_{\bf{w}} f_{i}(\mathbf{w}_{i}^{t-1})\rangle ]
	\end{align*}
}%
%$\mathbb{E}[\langle \nabla f(\overline{\mathbf{x}}^{t-1}), \frac{1}{N} \sum_{i=1}^{N} \mathbf{G}_{i}^{t}\rangle] 
%= \mathbb{E}[\mathbb{E}[\langle \nabla f(\overline{\mathbf{x}}^{t-1}), \frac{1}{N} \sum_{i=1}^{N} \mathbf{G}_{i}^{t}\rangle | \boldsymbol{\zeta}^{[t-1]}]] = \mathbb{E}[\langle \nabla f(\overline{\mathbf{x}}^{t-1}), \frac{1}{N} \sum_{i=1}^{N} \mathbb{E}[\mathbf{G}_{i}^{t}| \boldsymbol{\zeta}^{[t-1]}]\rangle ] = \mathbb{E}[\langle \nabla f(\overline{\mathbf{x}}^{t-1}), \frac{1}{N} \sum_{i=1}^{N} \nabla f_{i}(\mathbf{x}_{i}^{t-1})\rangle ]$
where the first equality follows by the law of expectations, the second equality follows because ${\mathbf{w}}^{t-1}$ is determined by $\boldsymbol{\xi}^{[t-1]}= [\boldsymbol{\xi}^{1}, \ldots, \boldsymbol{\xi}^{t-1}]$ and the third equality follows from $\mathbb{E}[\mathbf{g}_{i}^{t} | \boldsymbol{\xi}^{[t-1]}] = \mathbb{E}[\nabla F_{i}(\mathbf{w}_{i}^{t-1};\xi^{t}_{i}) | \boldsymbol{\xi}^{[t-1]}] = \nabla f_{i}(\mathbf{w}_{i}^{t-1})$.

Substituting Eqn.~\eqref{eq:w-difference-squre} and Eqn.~\eqref{eq:inner_product_w} into Eqn.~\eqref{eq:equal_1_total}, we have  %{\small $\mathbb{E}[f(\overline{\mathbf{x}}^{t})] \leq \mathbb{E}[f(\overline{\mathbf{x}}^{t-1})] - \frac{\gamma - \gamma^{2}L}{2} \mathbb{E} [\Vert \frac{1}{N} \sum_{i=1}^{N} \nabla f_{i} (\mathbf{x}_{i}^{t-1})\Vert^{2}] - \frac{\gamma}{2} \mathbb{E}[\Vert \nabla f(\overline{\mathbf{x}}^{t-1})\Vert^{2}]  + \frac{\gamma}{2}  \mathbb{E}[
%\Vert \nabla f(\overline{\mathbf{x}}^{t-1}) - \frac{1}{N} \sum_{i=1}^{N} \nabla f_{i} (\mathbf{x}_{i}^{t-1}) \Vert^{2} ] + \frac{L}{2N} \gamma^{2} \sigma^{2} \overset{(a)}{\leq} \mathbb{E}[f(\overline{\mathbf{x}}^{t-1})]  - \frac{\gamma - \gamma^{2}L}{2} \mathbb{E}[\Vert \frac{1}{N} \sum_{i=1}^{N} \nabla f_{i} (\mathbf{x}_{i}^{t-1})\Vert^{2}] - \frac{\gamma}{2} \mathbb{E} [\Vert \nabla f(\overline{\mathbf{x}}^{t-1})\Vert^{2}] + 2\gamma^{3} I^{2} G^{2} L^2 + \frac{L}{2N} \gamma^{2} \sigma^{2}$}
{\begin{align} \label{eq:11}
&\mathbb{E}[f({\mathbf{w}}^{t})] \nonumber\\
\leq &\mathbb{E}[f({\mathbf{w}}^{t-1})] - \frac{\gamma - \gamma^{2}\beta}{2} \mathbb{E} [\Vert \frac{1}{N} \sum_{i=1}^{N} \nabla_{\bf{w}} f_{i} (\mathbf{w}_{i}^{t-1})\Vert^{2}] \nonumber \\
 &- \frac{\gamma}{2} \mathbb{E}[\Vert \nabla_{\bf{w}} f({\mathbf{w}}^{t-1})\Vert^{2}+ \frac{\beta\gamma^{2}\sum\limits_{j = 1}^{L} {\sigma _j^2}}{2N}  \nonumber\\ &+ \frac{\gamma}{2}\mathbb{E}[||{\nabla _{{{\bf{w}}}}}f({{\bf{w}}^{t - 1}}) -\frac{1}{N}\sum\limits_{i = 1}^N {{\nabla _{{{\bf{w}}}}}} {f_i}({\bf{w}}_i^{t - 1})|{|^2}]\nonumber \\
\overset{(a)}{\leq}& \mathbb{E}[f({\mathbf{w}}^{t-1})] - \frac{\gamma}{2} \mathbb{E}[\Vert \nabla_{\bf{w}} f({\mathbf{w}}^{t-1})\Vert^{2}+ \frac{\beta\gamma^{2}\sum\limits_{j = 1}^{L} {\sigma _j^2}}{2N}  \nonumber\\ &+ \frac{\gamma}{2}\mathbb{E}[||{\nabla _{{{\bf{h}}_c}}}f({{\bf{h}}_c^{t - 1}}) -\frac{1}{N}\sum\limits_{i = 1}^N {{\nabla _{{{\bf{h}}_c}}}} {f_i}({\bf{h}}_{c,i}^{t - 1})|{|^2}]\nonumber \\
&+\frac{\gamma}{2}\mathbb{E}[||{\nabla _{{{\bf{h}}_s}}}f({{\bf{h}}_s^{t - 1}}) - \frac{1}{N}\sum\limits_{i = 1}^N {{\nabla _{{{\bf{h}}_s}}}} {f_i}({\bf{h}}_{s,i}^{t - 1})|{|^2}] \nonumber\\
\overset{(b)}{\leq} &\mathbb{E}[f({\mathbf{w}}^{t-1})] - \frac{\gamma}{2} \mathbb{E}[\Vert \nabla_{\bf{w}} f({\mathbf{w}}^{t-1})\Vert^{2}]+ \frac{\beta\gamma^{2}\sum\limits_{j = 1}^{L} {\sigma _j^2}}{2N}  \nonumber\\ &+ {\mathbbm{1}}_{\{I > 1\}} 2\beta^2\gamma^{3} I^{2} \sum\limits_{j = 1}^{L_c} {G_j^2} 
\end{align}
}%
where (a) follows from $0 < \gamma \leq \frac{1}{\beta}$ and (b) holds because of the following inequality~\eqref{difference_wc} and~\eqref{difference_ws}
{\begin{align}\label{difference_wc}
&\mathbb{E}[ \Vert \nabla_{{{\bf{h}}_c}} f({\mathbf{h}}_c^{t-1}) - \frac{1}{N} \sum_{i=1}^{N} \nabla_{{{\bf{h}}_c}} f_{i} (\mathbf{h}_{c,i}^{t-1})\Vert^{2}] \nonumber \\
 =& \mathbb{E} [ \Vert \frac{1}{N} \sum_{i=1}^{N}\nabla_{{{\bf{h}}_c}} f_{i}({\mathbf{h}}_c^{t-1}) - \frac{1}{N} \sum_{i=1}^{N} \nabla_{{{\bf{h}}_c}} f_{i} (\mathbf{h}_{c,i}^{t-1})\Vert^{2}] \nonumber \\
=& \frac{1}{N^{2}} \mathbb{E} [\Vert \sum_{i=1}^{N} \big( \nabla_{{{\bf{h}}_c}} f_{i}({\mathbf{h}}_c^{t-1}) - \nabla f_{{\mathbf{h}}_c} (\mathbf{h}_{c,i}^{t-1}) \big)\Vert^{2}] \nonumber \\
\leq& \frac{1}{N} \mathbb{E} [ \sum_{i=1}^{N} \Vert \nabla_{{{\bf{h}}_c}} f_{i}({\mathbf{h}}_c^{t-1}) - \nabla_{{{\bf{h}}_c}} f_{i} (\mathbf{h}_{c,i}^{t-1})\Vert ^{2}] \nonumber \\
\leq& \beta^2\frac{1}{N} \sum_{i=1}^{N}\mathbb{E}[ \Vert {\mathbf{h}}_c^{t-1} - \mathbf{h}_{c,i}^{t-1}\Vert^{2}] \nonumber \\
\leq&  {\mathbbm{1}}_{\{I > 1\}} 4\beta^2\gamma^{2} I^{2} \sum\limits_{j = 1}^{L_c} {G_j^2}
\end{align}
}%
where the first inequality follows by using $\Vert \sum_{i=1}^{N} \mathbf{z}_{i}\Vert^{2} \leq N \sum_{i=1}^{N} \Vert \mathbf{z}_{i}\Vert^{2}$ for any vectors $\mathbf{z}_{i}$; the second inequality follows from the smoothness of each $f_{i}$ by {\bf Assumption \ref{asp:1}}; and the third inequality follows from {\bf Lemma \ref{lm:diff-avg-per-node}}. Moreover, we have
{ \begin{align}\label{difference_ws}
&\mathbb{E}[ \Vert \nabla_{{{\bf{h}}_s}} f({\mathbf{h}}_s^{t-1}) - \frac{1}{N} \sum_{i=1}^{N} \nabla_{{{\bf{h}}_s}} f_{i} (\mathbf{h}_{s,i}^{t-1})\Vert^{2}] \nonumber \\
\leq& \beta^2\frac{1}{N} \sum_{i=1}^{N}\mathbb{E}[ \Vert {\mathbf{h}}_s^{t-1} - \mathbf{h}_{s,i}^{t-1}\Vert^{2}] \overset{(a)}{=}0 
\end{align}
where (a) holds because the server-side sub-models are aggregated in each training round (i.e., $I = 1$). Therefore, at any training round $t$, the server-side sub-model of each edge device is the aggregated version of server-side sub-models.

Dividing the both sides of Eqn.~\eqref{eq:11} by $\frac{\gamma}{2}$ and rearranging terms yields
{ \small \begin{align}
&\mathbb{E}\left [\Vert \nabla_{\bf{w}} f({\mathbf{w}}^{t-1})\Vert^{2}\right] \nonumber \\
\leq& \frac{2}{\gamma}\! \left(\mathbb{E}\left[f({\mathbf{w}}^{t-1})\right] \! - \!\mathbb{E}\left[f({\mathbf{w}}^{t})\right]\right) \!+\!\frac{\beta\gamma \!\sum\limits_{j = 1}^{L} {\sigma_j^2} }{N}  \nonumber \!+\! {\mathbbm{1}}_{\{I > 1\}} 4\beta^2\gamma^{2} I^{2}\! \sum\limits_{j = 1}^{L_c} \! {G_j^2} \label{eq:pf-thm-rate-eq8}
\end{align}
}%
Summing over $t\in\{1,\ldots, R\}$ and dividing both sides by $R$ yields
{\small 
\begin{align*}
&\frac{1}{R} \sum_{t=1}^{R} \mathbb{E}\left [\Vert \nabla_{\bf{w}} f({\mathbf{w}}^{t-1})\Vert^{2}\right] \\
\leq &\frac{2}{\gamma R} \!\left(f({\mathbf{w}}^{0}) - \mathbb{E}\left[f({\mathbf{w}}^{R})\right]\right) \!+\!\frac{\beta\gamma\! \sum\limits_{j = 1}^{L} {\sigma_j^2} }{N}  \nonumber \!+\! {\mathbbm{1}}_{\{I > 1\}} 4\beta^2\gamma^{2} I^{2} \sum\limits_{j = 1}^{L_c} \!{G_j^2} \\
\overset{(a)}{\leq} &  \frac{2}{\gamma R} \left(f({\mathbf{w}}^{0}) -f^{\ast}\right) +\!\frac{\beta\gamma \sum\limits_{j = 1}^{L} {\sigma_j^2} }{N}  \nonumber \!+\! {\mathbbm{1}}_{\{I > 1\}} 4\beta^2\gamma^{2} I^{2} \sum\limits_{j = 1}^{L_c} {G_j^2}
\end{align*}}
where (a) follows because $f^{\ast}$ is the minimum value of problem {\eqref{minimiaze_loss_function}}.

\section*{{C. Proof of Theorem 2}}\label{cc}

We first conduct a functional analysis for the objective function in problem~\eqref{subproblem_1}. Let $\Theta' \left( I \right) = \frac{{2\vartheta \left\{ {aI + b} \right\}}}{{\gamma I\left( {c - 4{\beta ^2}{\gamma ^2}{I^2}{T_1}} \right)}}$, where $a = {T_3} + T_s^F + T_s^B + {T_4}$, $b = {T_5} + {T_6}$ and $c = \varepsilon  - \frac{{\beta \gamma \sum\limits_{j = 1}^L {\sigma _j^2} }}{N}$. Taking the first-order derivative of $\Theta' \left( I \right)$ yields

\begin{equation}\label{accuracy_cons}
\frac{{\partial \Theta' \left( I \right)}}{{\partial I}} = \frac{{2\vartheta }}{\gamma }\frac{{\Xi \left( I \right)}}{{{{\left( {cI - 4{\beta ^2}{\gamma ^2}{I^3}{T_1}} \right)}^2}}}
\end{equation}
where
\begin{equation}\label{E_I}
\Xi \left( I \right) = 8a{\beta ^2}{\gamma ^2}{I^3}{T_1} + 12b{\beta ^2}{\gamma ^2}{I^2}{T_1} - bc.
\end{equation}
Since $\frac{{\partial \Xi \left( I \right)}}{{\partial I}} = 24{\beta ^2}{\gamma ^2}I{T_1}\left( {aI + b} \right) > 0$, $\Xi \left( I \right)$ is an increasing function of $I$. Considering that $\Xi \left( 0 \right) =  - bc < 0$ and $\mathop {\lim }\limits_{I \to  + \infty } \Xi \left( I \right) =  + \infty  > 0$, there must exist $I'$ satisfying $\Xi \left( I' \right) = 0$ for $I' \in \left( {0, + \infty } \right)$, which can be easily obtained by classical Newton-Raphson method. Then, we see that $\frac{{\partial \Theta' \left( I \right)}}{{\partial I}} \le 0$ for $I \in ( 0, I'] $ and  $\frac{{\partial \Theta' \left( I \right)}}{{\partial I}} \ge 0$ for $I \in \left( {I', + \infty } \right)$, which means that the objective function decreases and then increases with respective to $I$ and thus reaches a minimum at $I = I'$.

Based on constraint $\mathrm{C4}$ and the characteristics of the objective function analyzed above, it is obvious that the optimal value $I^*$ only exists on both sides of $I'$, i.e.,
\begin{equation}\label{accuracy_cons}
{I^*} = \left\{ {\begin{array}{*{20}{c}}
1&{I' \le 1}\\
{\arg {{\min }_{I \in \left\{ {\left\lfloor {I'} \right\rfloor ,\left\lceil {I'} \right\rceil } \right\}}}\Theta' \left( I \right)}&{I' > 1}
\end{array}}, \right.
\end{equation}
where $\left\lfloor {\cdot} \right\rfloor$ and $\left\lceil {\cdot} \right\rceil$ denote floor and ceiling operations.

\ifCLASSOPTIONcaptionsoff
  \newpage
\fi



%\bibliographystyle{IEEEtran}


% \end{thebibliography}

\bibliographystyle{IEEEtran}
\bibliography{reference}

\end{document}